\documentclass[lettersize,journal]{IEEEtran}
\usepackage{amsmath,amsfonts}
\usepackage{algorithmic}
\usepackage{array}
\usepackage[font=small, labelfont=bf]{subfig}
\usepackage{textcomp}
\usepackage{url}
\usepackage{verbatim}
\usepackage{graphicx}

\usepackage{multicol}
\usepackage[bookmarks=true]{hyperref}
\usepackage{graphics} % for pdf, bitmapped graphics files
\usepackage{epsfig} % for postscript graphics files
\usepackage{mathrsfs}
\usepackage[font=small,labelfont=bf]{caption}
\usepackage{times} % assumes new font selection scheme installed
\usepackage{amssymb}  % assumes amsmath package installed
\usepackage[ruled,vlined,linesnumbered]{algorithm2e}
\usepackage{optidef}
\usepackage{booktabs}
\usepackage{amsthm}
\usepackage{svg}
\usepackage{cite}
\usepackage{hyperref}
\usepackage{cleveref}
\usepackage{dsfont}
\usepackage{graphics}
\usepackage{float}
\usepackage{flushend}

\newtheorem{assumption}{Assumption}
\newtheorem{theorem}{Theorem}
\newtheorem{lemma}{Lemma}
\newtheorem{example}{Example}

\newtheorem{definition}{Definition}
\newcommand{\iidsim}{\overset{\mathrm{i.i.d.}}{\sim}}
\hypersetup{
    colorlinks=true,
    linkcolor=blue,
    filecolor=green,      
    urlcolor=black,
}

\newcommand{\qu}{{q^\mathrm{u}}}
\newcommand{\qa}{{q^\mathrm{a}}}

\newcommand{\dqu}{{\delta q^\mathrm{u}}}
\newcommand{\dqa}{{\delta q^\mathrm{a}}}
\newcommand{\dq}{{\delta q}}
\newcommand{\Ka}{\mathbf{K}_\mathrm{a}}
\newcommand{\ka}{k_\mathrm{a}}
\newcommand{\Mu}{\mathbf{M}_\mathrm{u}}
\newcommand{\Ba}{\mathbf{B}^\mathrm{a}}
\newcommand{\Bu}{\mathbf{B}^\mathrm{u}}

\newcommand{\tauU}{\tau^\mathrm{u}}
\newcommand{\tauA}{\tau^\mathrm{a}}
\newcommand{\J}{\mathbf{J}}
\newcommand{\Ju}[1][]{\mathbf{J}_{\mathrm{u}_{#1}}}
\newcommand{\Ja}[1][]{\mathbf{J}_{\mathrm{a}_{#1}}}
\newcommand{\Jn}[1][]{\mathbf{J}_{\mathrm{n}_{#1}}}
\newcommand{\Jt}[1][]{\mathbf{J}_{\mathrm{t}_{#1}}}
\newcommand{\JuActive}{\tilde{\mathbf{J}}_\mathrm{u}}
\newcommand{\JaActive}{\tilde{\mathbf{J}}_\mathrm{a}}
\newcommand{\R}[1][]{\mathbb{R}^{#1}}
\newcommand{\nU}{n_\mathrm{u}}
\newcommand{\nA}{n_\mathrm{a}}
\newcommand{\nC}{n_\mathrm{c}}
\newcommand{\nCG}{n_\mathrm{cg}}

\newcommand{\norm}[1]{\left\lVert{#1}\right\rVert}

\newcommand{\minimize}{\mathrm{min}.}
\newcommand{\DfDx}[2]{\frac{\partial {#1}}{\partial {#2}}}
\newcommand{\DfDxLine}[2]{\partial {#1} / \partial {#2}}
\newcommand{\A}{\mathbf{A}}
\newcommand{\B}{\mathbf{B}}
\newcommand{\I}{\mathbf{I}}
\newcommand{\code}[1]{{\fontfamily{cmss}\selectfont {#1}}}

\newcommand{\Nearest}{\mathtt{Nearest}}
\newcommand{\Extend}{\mathtt{Extend}}
\newcommand{\ContactSample}{\mathtt{ContactSample}}

\def\BibTeX{{\rm B\kern-.05em{\sc i\kern-.025em b}\kern-.08em
    T\kern-.1667em\lower.7ex\hbox{E}\kern-.125emX}}
\usepackage{balance}

\begin{document}
\title{Global Planning for Contact-Rich Manipulation via Local Smoothing of Quasi-dynamic Contact Models}

\author{Tao Pang$^{*}$, H.J. Terry Suh$^{*}$, Lujie Yang and Russ Tedrake% <-this % stops a space
\thanks{*These authors contributed equally to this work.}% <-this % stops a space
\thanks{ This work is funded by Amazon PO 2D-06310236, Lincoln Laboratory/Air Force PO 7000470769, NSF Award No. EFMA-1830901 and the Ocado Group. The authors are with the Computer Science and Artificial Intelligence Laboratory (CSAIL), Massachusetts Institute of Technology, USA. {\tt\small \{pangtao, hjsuh94, lujie, russt\}@csail.mit.edu}}%
}

% \markboth{Journal of \LaTeX\ Class Files,~Vol.~18, No.~9, September~2020}%
% {How to Use the IEEEtran \LaTeX \ Templates}

\maketitle

\begin{abstract}
The empirical success of Reinforcement Learning (RL) in contact-rich manipulation leaves much to be understood from a model-based perspective, where the key difficulties are often attributed to (i) the explosion of contact modes, (ii) stiff, non-smooth contact dynamics and the resulting exploding / discontinuous gradients, and (iii) the non-convexity of the planning problem.
The stochastic nature of RL addresses (i) and (ii) by effectively sampling and averaging the contact modes.
On the other hand, model-based methods have tackled the same challenges by smoothing contact dynamics analytically.
Our first contribution is to establish the theoretical equivalence of the two smoothing schemes for simple systems, and provide qualitative and empirical equivalence on several complex examples. 
In order to further alleviate (ii), our second contribution is a convex, differentiable and quasi-dynamic formulation of contact dynamics, which is amenable to both smoothing schemes, and has proven to be highly effective for contact-rich planning.
Our final contribution resolves (iii), where we show that classical sampling-based motion planning algorithms can be effective in global planning when contact modes are abstracted via smoothing. Applying our method on several challenging contact-rich manipulation tasks, we demonstrate that efficient model-based motion planning can achieve results comparable to RL, but with dramatically less computation.
\newline \href{https://sites.google.com/view/global-planning-contact/home}{\color{magenta} \code{https://sites.google.com/view/global-planning-contact/home}}
\end{abstract}


\section{Introduction}
\IEEEPARstart{R}{ecent} advances in RL have shown impressive results in manipulation through contact-rich dynamics that were difficult to achieve with previous model-based methods \cite{andrychowicz2020learning, chen2022system, nagabandi2019deep}. However, it is yet unclear what made these methods successful where model-based methods have struggled. Our high-level goal is to understand these reasons and interpret them from a model-based point of view. From such interpretations, we seek to apply the ingredients behind the empirical success of RL with the generalizability and efficiency of models.

From a model-based perspective, the most significant obstacle for planning through contact lies in the \emph{hybrid} nature of contact dynamics (i.e. numerous modes of smooth dynamics separated by guard surfaces). The non-smooth nature of the resulting dynamics implies that Taylor approximation no longer holds locally; thus, the locally linear model constructed with the gradient quickly becomes invalid. This invalidity of the local model presents significant challenges for both iterative gradient-based optimization \cite{clarke,gradientsampling}, and sampling-based planning that rely on local distance metrics \cite{shkolnik2009reachability,wu2020r3t,haddad2021anytime}.

\begin{figure}[t]
	\centering\includegraphics[width = 0.49\textwidth]{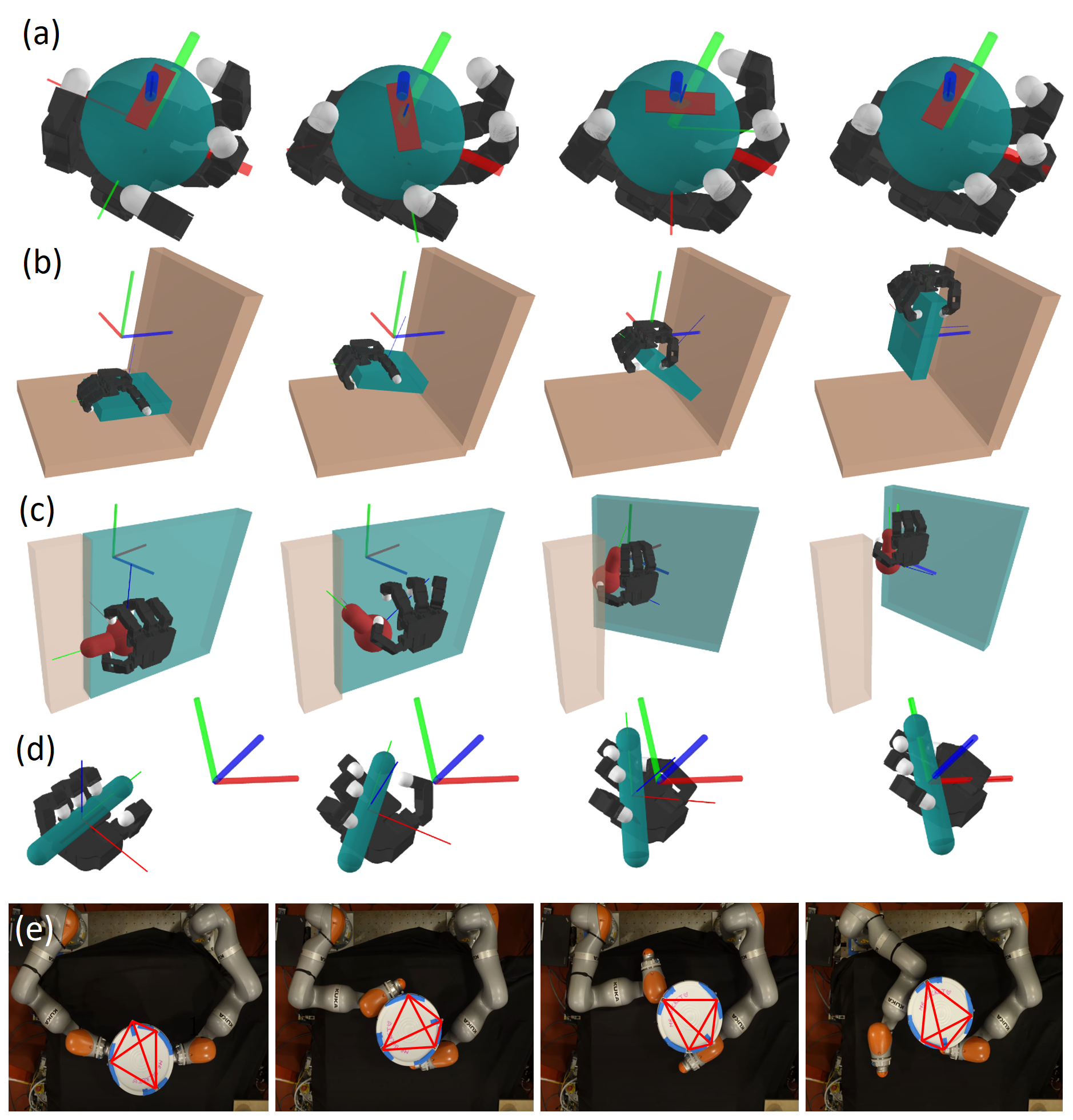}
    \caption{Examples of contact-rich plans generated by our method. Each row corresponds to a time-snapshot of five different tasks: (\textbf{a}) 3D in-hand manipulation, (\textbf{b}) plate pickup using extrinsic dexterity, (\textbf{c}) door opening with a hinge and (\textbf{d}) pen placement, (\textbf{e}) bucket rotation with two iiwa arms on hardware. Each plan is generated online in the order of a minute of wall-clock time.}
	\label{fig:banner}
	\vskip -0.25 true in
\end{figure}

Faced with such challenges, many existing works have sought to explicitly consider contact modes by either enumerating or sampling them. With model-based knowledge of the dynamic modes, these planners typically alternate between continuous-state planning in the current contact mode and a discrete search for the next mode, resulting in trajectories punctuated by a handful of mode changes \cite{wu2020r3t, chen2021trajectotree, cheng2021contact}. Another example of such approaches comes from transcribing the modes into binary variables, resulting in Mixed-Integer Programs \cite{aceituno2020global,marcucci2017approximate, hogan2020feedback}. However, both of these methods suffer from poor scalability: complex systems such as 3D in-hand manipulation can have trillions of contact modes \cite{huang2020efficient}. The explosion of contact modes indicates that a viable planning strategy in the regime of rich contacts needs to somehow avoid considering all possible contact sequences \cite{posa2014direct}. 

Indeed, recent works have attributed the success of RL to its stochastic nature, where the contact modes are stochastically abstracted \cite{bundledgradients,suh2022differentiable,lidec2022leveraging} by a process of sampling and averaging. This procedure, termed \emph{randomized smoothing}, stochastically smooths the underlying landscape \cite{duchi2}. These works show that randomized smoothing can locally abstract neighboring contact modes into a single average linear model, implicitly creating a \emph{stochastic force field} where contact is applied from a distance by means of randomization. This average model can give more informative gradients that are useful for planning.

However, the idea of smoothing contact modes for planning is hardly new in model-based literature. Many of the previous works in planning through contact have employed schemes to iteratively smooth contact dynamics in order to improve convergence of the solvers \cite{posa2014direct,mordatch2012contact,manchester2020variational,huang2021plasticinelab,howell2022dojo}, which leads to similar force-from-a-distance interpretations. Importantly, these methods differ from randomized smoothing as they consider the structure of the underlying equations to produce smooth approximations. We term this process \emph{analytic smoothing}. The presence of the two smoothing schemes, randomized and analytic, begs an important question: what differences, both in terms of theoretical characterizations and empirical performance, can we expect from both smoothing schemes?

Our first contribution is to establish the theoretical equivalence of the two smoothing schemes for simple systems under our framework (Sec.\ref{sec:localsmoothing},\ref{sec:smoothingequivalence}). Using this framework, we further show how to efficiently compute the locally linear models (i.e. gradients) of the smoothed dynamics online (Sec.\ref{sec:smoothing_computation_schemes}), and show that both smoothing schemes result in comparable qualitative characteristics, as well as empirical performance on a number of complex examples (Sec.\ref{sec:traj_opt},\ref{sec:rrt_results})

Furthermore, at the heart of the model-based planning approach we take lies the question: what is the \emph{right} model for gradient-based contact-rich manipulation planning? 
We believe this model needs to be 
(\textbf{i}) numerically robust
and (\textbf{ii}) differentiable, so that local linearizations can be efficiently computed for planning. 
Most importantly, the model should be able to (\textbf{iii}) predict \emph{long-term} behavior, so that the planner can look far ahead while taking few steps.
Lastly, the model should be (\textbf{iv}) amenable to smoothing in order to provide more informative gradients across contact modes. 

Our second contribution (Sec.\ref{sec:quasidynamic}) is a contact dynamics model that satisfies all requirements. Specifically, we propose a \emph{convex} implicit time-stepping contact model (Sec.\ref{sec:convex_quasi_dynamic_contact_dynamics}). Convexity is achieved by Anitescu's relaxation of frictional contact constraints \cite{anitescu2006optimization}, which in practice only introduces mild non-physical behavior \cite{pang2021convex, castro2021unconstrained}. The convexity provides a clear numerical advantage over the traditional Linear Complementarity Problem (LCP) formulation \cite{anitescu1997formulating, stewart2000rigid}. Moreover, the derivatives of the dynamics with respect to the current state and control action can be readily computed using sensitivity analysis for convex conic programs \cite{agrawal2019differentiatingcone} (Sec.\ref{sec:quasi_dynamic_derivatives}). 

For long-term predictability, we adopt the quasi-dynamic assumption widely used in robotic manipulation \cite{mason2001mechanics, motioncones, aceituno2020global, cheng2021contact}. 
A quasi-dynamic model ``sees'' further than its second-order counterpart by ignoring transient dynamics and focuses on transitions between steady-states. Furthermore, by throwing away kinetic energy at every time step, quasi-dynamic models are also simpler as they do not need variables for velocity, nor parameters to describe damping.

We validate and test our quasi-dynamic contact model by running the same input trajectories in a high-fidelity second-order simulator Drake \cite{drake}, as well as on hardware (Sec.\ref{sec:sim2real}). Our results show that our model is able to approximate the second-order dynamics well if the system considered is highly damped and truly dominated by frictional forces.

Our contact model can also be analytically smoothed out using a log-barrier relaxation (Sec.\ref{sec:analyticsmoothing}). In this relaxation scheme, the hard contact constraints are softly enforced by a log-barrier function, a common technique used in the interior-point method for convex programs \cite[\textsection 11]{boyd2004convex} \cite{howell2022dojo}.
We further show that the gradients of the smoothed contact model can be easily computed with the implicit function theorem.

Finally, we believe that another key factor behind the empirical success of RL lies in its goal of performing global optimization with stochasticity. In contrast, deterministic model-based optimization for planning through nonlinear dynamics generally results in non-convex optimization problems, where the quality of different local minima can be a make-or-break factor. Solving such problems requires non-trivial initialization and tuning \cite{onol2020tuning}, both of which can be highly problem specific and notoriously hard to debug. This motivates us to search further for a more global approach.

In robotics, sampling-based motion planning (SBMP) algorithms such as the Rapidly-Exploring Random Tree (RRT) \cite{lavalle1998rapidly} are widely used for global search problems, including those with kinodynamic constraints \cite{karaman2010optimal}. However, the success of such planners has rarely extended to contact-rich settings. We find that while optimization-based methods for planning through contact have employed smoothing schemes, all existing SBMP methods for contact planning explicitly consider modes instead of smoothing them \cite{cheng2021contact,wu2020r3t,chen2021trajectotree,motioncones,terry}, as SBMP methods do not inherently require local characterizations of dynamics (i.e. gradients). Yet, previous works have shown that such local models are highly relevant for designing more efficient distance metrics during the nearest neighbor queries that respects dynamic reachability \cite{shkolnik2009reachability,wu2020r3t,haddad2021anytime}.

Our final contribution is to fill in this gap by combining smoothing-based contact mode abstraction and the global search capabilities of RRT. We enable RRT to search through contact dynamics constraints by utilizing a novel distance metric derived from the local smoothed models (Sec.\ref{sec:mahalanobis}). In addition, we propose an efficient extension step by computing actions that expand the tree using the local models (Sec.\ref{sec:rrt_for_contact}). With a variety of contact-rich tasks inspired by \cite{rajeswaran2018learning} that involve both intrinsic and extrinsic dexterity \cite{extrinsic}, we show that combining smoothing with RRT achieves tractable global motion planning for highly contact-rich manipulation (Sec.\ref{sec:rrt_results}). To the best of our knowledge, our work appears to be the first to successfully combine SBMP with contact mode smoothing.

\textbf{Summary of Contributions}. (\textbf{i}) We establish the theoretical equivalence of methods for randomized and analytic smoothing on simple systems, and qualitative / empirical equivalence on complex systems. (\textbf{ii}) We present a convex, differentiable formulation of quasi-dynamic contact dynamics and its analytic smoothing, which we show to be highly effective for contact-rich manipulation planning. (\textbf{iii}) We combine contact mode smoothing with sampling-based motion planning, filling in a gap in the spectrum of existing methods and achieving efficient global planning through highly-rich contact dynamics. 

\section{Local Theory of Smoothing}\label{sec:localsmoothing}
\noindent Before discussing complex systems with contact, we formalize mathematically what it means to smooth a function, as well as different algorithms to compute their local approximations. We aim to provide a coherent view of the different smoothing schemes and their equivalence relations.

Consider a locally-Lipschitz (but potentially non-smooth) function $f:\mathbb{R}^n\rightarrow\mathbb{R}^m$. We consider a class of iterative algorithms that rely on finding a locally-affine approximation to $f$ at every iteration, around the current operating point $\bar{x}$. We denote this affine model as $g(\cdot;\bar{x})$,
\begin{equation}
    g(x;\bar{x}) \coloneqq \mathbf{J}(\bar{x})\underbrace{(x - \bar{x})}_{=:\delta x} + \mu(\bar{x}).
\end{equation}
which consists of the sensitivity term $\mathbf{J}(\bar{x})\in\mathbb{R}^{m\times n}$, and a bias term $\mu(\bar{x})\in\mathbb{R}^m$, which we refer to as model parameters $\left(\mathbf{J}\left(\bar{x}\right),\mu\left(\bar{x}\right)\right)$. 

If $f\in C^1$ (i.e. smooth), then a natural candidate to use for the model parameters is the first-order Taylor expansion of $f$ around $\bar{x}$, also known as the \emph{linearization} of the function. The linearization requires obtaining the gradient of $f$.

\begin{definition}\label{def:linearization}\normalfont
    We define the \emph{\emph{linearization}} of $f$ around the nominal point $\bar{x}$ as $\bar{f}$,
    \begin{equation}
        \label{eq:f_bar}
        \begin{aligned}
            \bar{f}(x;\bar{x}) \coloneqq \frac{\partial f}{\partial x}\bigg|_{x=\bar{x}}(x-\bar{x}) + f(\bar{x}).
        \end{aligned}
    \end{equation}
\end{definition}
\noindent Note that $\bar{f}$ is a locally-affine model with parameters $(\partial f/\partial x|_{x=\bar{x}},f(\bar{x}))$.

However, when $f$ is non-smooth with fast-changing discontinuous gradients, the Taylor expansion \eqref{eq:f_bar} quickly deviates from $f$. Therefore, approaches that utilize \eqref{eq:f_bar} (e.g. approaches that use the gradient) often suffer greatly. In this section, we investigate a set of smoothing methods that can give more informative locally-affine approximations to $f$. 

\subsection{Smooth Surrogate of a Function}
We first formalize the process of smoothing a function $f$ by introducing the \emph{smooth surrogate}, which is a modification of $f$ that closely resembles the original function, yet is smooth.

\subsubsection{The Smooth Surrogate}
We define a smooth surrogate using the process of convolution.
\begin{definition}\label{def:surrogate}\normalfont
    Given a probability distribution $\rho$ characterized by a smooth density function, we define $f_\rho$ as the \underline{smooth surrogate} of $f$,
    \begin{equation}
        \label{eq:smooth_surrogate_def}
        f_\rho(x) \coloneqq \mathbb{E}_{w\sim\rho}[f(x + w)], 
    \end{equation}
\end{definition}

The smooth surrogate provides us with a natural formalism of common smooth approximations that are made to non-smooth functions. We give some examples of smooth surrogates that are commonly used in machine learning and robotics, as well as their corresponding densities, in Table \ref{table:examples}.

\begin{table}[thpb]
\centering
\begin{tabular}{||c | c | c||} 
 \hline
 $f$ & $\rho$ & $f_\rho$ \\ [0.5ex] 
 \hline\hline
 \href{https://en.wikipedia.org/wiki/Rectifier_(neural_networks)}{\texttt{ReLU}} & \href{https://en.wikipedia.org/wiki/Logistic_distribution}{\texttt{Logistic}} & \href{https://pytorch.org/docs/stable/generated/torch.nn.Softplus.html}{\texttt{SoftPlus}} \\
 \href{https://en.wikipedia.org/wiki/Heaviside_step_function}{\texttt{Heaviside}} & \href{https://en.wikipedia.org/wiki/Cauchy_distribution}{\texttt{Cauchy}} & \href{https://en.wikipedia.org/wiki/Inverse_trigonometric_functions}{\texttt{ArcTan}} \\
  \href{https://en.wikipedia.org/wiki/Heaviside_step_function}{\texttt{Heaviside}} & \href{https://en.wikipedia.org/wiki/Logistic_distribution}{\texttt{Logistic}} & \href{https://en.wikipedia.org/wiki/Sigmoid_function}{\texttt{Sigmoid}} \\
 \hline
\end{tabular}
\caption{Examples of smooth surrogates of non-smooth functions.}
\label{table:examples}
\end{table}

Note that the expectation \eqref{eq:smooth_surrogate_def} is applied pointwise. Given a single point $x$, it is desirable to average points from a neighborhood that is centered at $x$. Thus, we make the following assumption on the distribution $\rho$, which we further assume to be smooth and elliptical to enforce regularity. 

\begin{assumption}\normalfont
    We assume $\rho$ is a zero-mean elliptical distribution with an integrable and smooth density: $\rho\in L^1,\rho\in C^1$.
\end{assumption}

The smoothness of $f_\rho$, despite the non-smoothness of $f$, is formalized as the following lemma which follows from properties of convolution. 

\begin{lemma}\normalfont
    If $\rho\in C^1$, then $f_\rho \in C^1$.
\end{lemma}
\begin{proof}
    We show that the gradients are well-defined using Leibniz integral rule,
    \begin{equation}
        \begin{aligned}
            \frac{d}{dx} f_\rho(x) & = \textstyle\frac{d}{dx} \int f(x+w)\rho(w)dw \\
                                   & = \textstyle\frac{d}{dx} \int f(y)\rho(y-x)dy \\
                                   & = \textstyle\int f(y) \left(\frac{d}{dx} \rho(y-x)\right) dy,
        \end{aligned}
    \end{equation}
    where we make the reparametrization $y=x+w$, and $\frac{d}{dx} \rho(y-x)$ is well-defined since $\rho\in C^1$.
\end{proof}

\subsubsection{Linearization of the Smooth Surrogate}

Since $f_\rho\in C^1$, first-order Taylor approximation now holds for $f_\rho$ and we can successfully use the linearization of $f_\rho$ as an affine model that best represents $f_\rho$ at an operating point of $\bar{x}$.
\begin{definition}\normalfont
We apply the definition of linearization on $f_\rho$ to define the linearization of the smooth surrogate, 
\begin{equation}
    \begin{aligned}
    \mathbf{J}_\rho(\bar{x}) & \coloneqq \textstyle\frac{\partial}{\partial x} f_\rho(\bar{x}) = \textstyle\frac{\partial}{\partial x}\mathbb{E}_{w\sim\rho}[f(x+w)]|_{x=\bar{x}} \\
    \mu_\rho(\bar{x}) & \coloneqq f_\rho(\bar{x}) = \mathbb{E}_{w\sim\rho}[f(\bar{x}+w)].
    \end{aligned}
\end{equation}
\end{definition}
We note that there are two alternative ways to characterize $\mathbf{J}_\rho$. First, the following lemma establishes the validity of the \emph{reparametrization gradient} of $f_\rho$ \cite{kingma}, which requires access to gradients of $f$.  
\begin{lemma}\normalfont\label{lem:reparametrization}
    $\mathbf{J}_\rho$ can be obtained by taking the expectation of gradients,
    \begin{equation}
    \mathbf{J}_\rho(\bar{x}) = \textstyle\mathbb{E}_{w\sim\rho}\left[\frac{\partial}{\partial x}f(x)\big|_{x=\bar{x}+w}\right].
    \end{equation}
    \begin{proof}
        The exchange of derivative and expectation follows from Leibniz integral rule. 
    \end{proof}
\end{lemma}

The next lemma establishes the validity of the \emph{score function} gradient of $f_\rho$, which does not requires access to gradients of $f$. This gradient, commonly used in RL, is also referred to as the likelihood ratio, or the REINFORCE gradient \cite{reinforce}.

\begin{lemma}\normalfont(\textbf{Stein's Lemma})\label{lem:stein}
    For any choice of $b$ that does not depend on $w$, $\mathbf{J}_\rho$ can be obtained as follows,
    \begin{equation}
    \mathbf{J}_\rho(\bar{x}) = \mathbb{E}_{w\sim \rho}\bigg[\left[f(\bar{x} + w)-b\right]\underbrace{\left[-\frac{\nabla \rho(w)}{\rho(w)}\right]^\intercal}_{=:S(w)}\bigg].
    \end{equation}
    \begin{proof}
        We refer to \cite{stein,stein2} for detailed proof for the case that $b=0$. If $b$ is non-zero, we use the fact that the expectation of the score function $S(w)$ is zero, $\mathbb{E}_w S(w)=0$ to show that the choice of $b$ has no impact on the equality.
    \end{proof}
\end{lemma}

\subsection{Rethinking Linearization as a Minimizer}

We also provide an alternative motivation of smoothing. Instead of interpreting smoothing as modifying $f$ directly, we can provide an alternative interpretation of linearization: what is the \emph{best} affine approximation of $f$ around a nominal point $\bar{x}$, and according to which metric? We propose a natural objective of the residual between the original function $f$ and the linear model $g$, weighted by some distribution $\rho$. Since $f_\rho$ is an affine model that has to minimize the residual from a distribution rather than a single point, we expect the affine model to be more informative compared to the exact linearization, especially when $f$ is non-smooth. We formalize this through the following theorem.

\begin{theorem}\label{thm:regression}\normalfont
Let $\rho$ be a zero-mean Gaussian with covariance $\mathbf{\Sigma}$, $\rho(w)=\mathcal{N}(w;0,\mathbf{\Sigma})$. Consider the problem of choosing the best affine model with parameters $(\mathbf{J},\mu)$ such that the residual distributed according to $\rho$ is minimized,
\begin{equation}\label{eq:residual}
    \underset{\mathbf{J},\mu}{\minimize} \; \frac{1}{2}\mathbb{E}_{w\sim\rho} \big[\|f(\bar{x}+w) - \mathbf{J} w - \mu \|^2_2\big].
\end{equation}
Then, the solution to the problem is given by the exact linearization of the smooth surrogate,
\begin{equation}
    \mathbf{J}^*=\mathbf{J}_\rho(\bar{x}),\quad \mu^* = \mu_\rho(\bar{x}).
\end{equation}
\end{theorem}
\noindent\begin{proof}
    (\ref{eq:residual}) is a linear regression problem and convex in the arguments, and thus, the first-order stationarity condition implies optimality. Letting $F(\mathbf{J},\mu)$ denote the objective function, the gradients are 
    \begin{subequations}
        \begin{align}\label{eq:stationarityone}
            \partial F/\partial \mu & = \mathbb{E}_{w}[f(\bar{x}+w)] - \mu^*, \\
            \label{eq:stationaritytwo}
            \partial F/\partial \mathbf{J} & = \mathbb{E}_{w}[ww^\intercal]\mathbf{J}^* - \mathbb{E}_w[f(\bar{x}+w)w^\intercal].
        \end{align}
    \end{subequations}
    Setting them to zero, \eqref{eq:stationarityone} directly leads to the solution for $\mu^*$. In addition, $\mathbf{J}^*$ can be obtained from \eqref{eq:stationaritytwo}:
    \begin{equation}
        \begin{aligned}
            \mathbf{J}^* & = \mathbb{E}_{w\sim\rho}[ww^\intercal]^{-1}\mathbb{E}_{w\sim\rho}[f(\bar{x}+w)w^\intercal] \\
            & = \textstyle\frac{\partial}{\partial x} \mathbb{E}_{w\sim\rho}[f(x+w)]|_{x=\bar{x}},
        \end{aligned}
    \end{equation}
    which comes from using Stein's lemma, where the score function of the Gaussian is $S(w) = \mathbf{\Sigma}^{-1} w$.
\end{proof}

Theorem \ref{thm:regression} offers another insight of why smoothing is necessary if we generalize the objective around a distribution. A slightly technical modification of this theorem, which replaces $\mathbf{J}w$ in \eqref{eq:residual} with $\mathbf{J}\mathbf{F}^{-1}(w)S(w)$ where $\mathbf{F}(w)$ is the Fisher information matrix $\mathbf{F}(w)\coloneqq\mathbb{E}_{w\sim\rho}[S(w)S(w)^\intercal]$, also applies for general elliptical distributions with identical proof. However, we document the Gaussian case for intuition.

\begin{example}\label{ex:theory}\normalfont
    To illustrate the different smoothing schemes and the implication of Theorem \ref{thm:regression}, we plot a simple example for a polynomial $f$ in Fig. \ref{fig:theoryexample}. The example shows that when least-squares fit is performed on the samples of $f(\bar{x}+w_i)$, the linear model that minimizes \eqref{eq:residual} converges to the exact linearization of the smooth surrogate $f_\rho$.
    \begin{figure}[thpb]
	\centering\includegraphics[width = 0.5\textwidth]{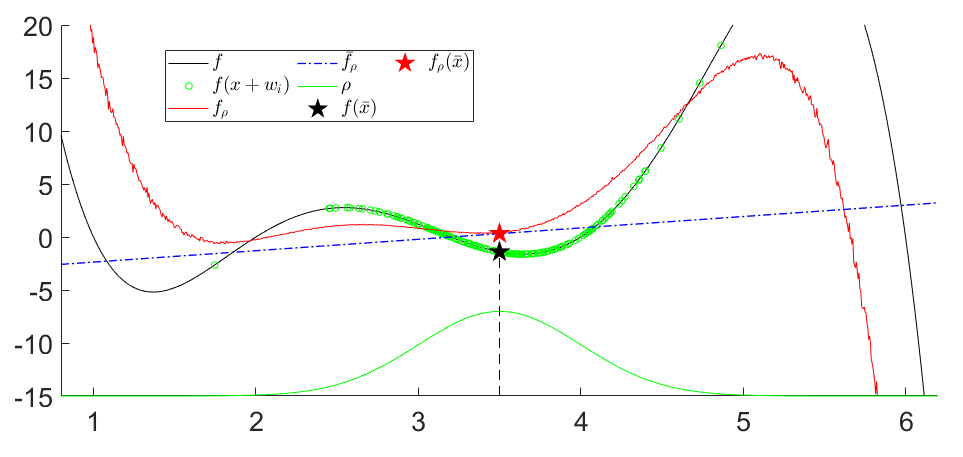}
    \caption{Polynomial $f$ (black) and its smooth surrogate $f_\rho$ (red), for the case where $\rho$ is a Gaussian (green). The plotted $f_\rho$ is obtained using Monte-Carlo estimation. When $f(\bar{x}+w_i)$ (green) are sampled around $\bar{x}$ and the least squares fit is performed (blue line, $\bar{f}_\rho$), it converges to the exact linearization of the smooth surrogate $f_\rho$.}
	\label{fig:theoryexample}
    \end{figure}
\end{example}

\subsection{Computation of the Locally Linear Model \label{sec:smoothing_computation_schemes}}
The equivalence between different characterizations of $(\mathbf{J}_\rho,\mu_\rho)$ implies different methods to compute them with access to $f$. We investigate three major methods to compute these quantities.

\subsubsection{Analytical Smoothing \label{sec:analytic_smoothing}} If $f$ and $\rho$ are sufficiently structured, one can analytically compute the convolution $f_\rho$ and its derivative $\J_\rho$, similar to the analytic derivatives of $f_\rho$ on simple functions from Table \ref{table:examples}.

\subsubsection{Randomized Smoothing, First-Order \label{sec:randomized_smoothing}}
If $\rho$ is a distribution that is easy to sample from, one can compute a Monte-Carlo estimate of $\mu_\rho$ by averaging the samples.
\begin{lemma}\normalfont
    The sample mean of $f(x + w_i)$, where $w_i$ are drawn i.i.d. from $\rho$, is an unbiased estimator of $\mu_\rho$,
    \begin{equation}
        \mathbb{E}_{w_i\iidsim\rho}\left[\frac{1}{N}\sum^N_{i=1}f(\bar{x} + w_i) \right] = \mu_\rho(\bar{x}).
    \end{equation}
\end{lemma}
\begin{proof}
    Law of large numbers.
\end{proof}
Furthermore, Lemma \ref{lem:reparametrization} implies that $\mathbf{J}_\rho$ can be estimated by sampling the gradients. We note that crucially, use of this method requires $f$ to be locally-Lipschitz (i.e. continuous). The failure mode of first-order randomized smoothing is documented in \cite{teg,bundledgradients,suh2022differentiable} and also briefly illustrated in Example \ref{ex:computation}.

\begin{lemma}\normalfont
    Under locally Lipschitz $f$, the sample mean of $\textstyle\frac{\partial f}{\partial x}(x + w_i)|_{x=\bar{x}}$, where $w_i$ are drawn i.i.d. from $\rho$, is an unbiased estimator of $\mathbf{J}_\rho$,
    \begin{equation}
        \mathbb{E}_{w_i\iidsim\rho}\left[\frac{1}{N}\sum^N_{i=1}\frac{\partial f}{\partial x}(\bar{x} + w_i) \right] = \mathbf{J}_\rho(\bar{x}).
    \end{equation}
\end{lemma}
\begin{proof}
    We use Law of large numbers then Lemma \ref{lem:reparametrization}.
\end{proof}

\subsubsection{Randomized Smoothing, Zeroth-Order} Finally, Lemma \ref{lem:stein} implies that we can estimate $\mathbf{J}_\rho$ in zeroth-order, which is analogous to how gradients are computed in RL \cite{reinforce, suh2022differentiable}.
\begin{lemma}\normalfont
    For any choice of $b$ that does not depend on $w$, we have 
    \begin{equation}
        \mathbb{E}_{w_i\iidsim\rho}\left[\frac{1}{N}\sum^N_{i=1} \left[f(\bar{x} + w_i) - b\right] S(w_i) \right] = \mathbf{J}_\rho(\bar{x}).
    \end{equation}
\end{lemma}
\begin{proof}
    We use Law of large numbers then Lemma \ref{lem:stein}.
\end{proof}
A typical choice of $b$ is the zero-noise evaluation of $f$, $b=f(\bar{x})$. Note that if $\rho$ is Gaussian, we can also take a finite-sample approximation to the least-squares problem:
\begin{equation}
\label{eq:generic_f_local_least_squares}
\begin{aligned}
\underset{\mathbf{J}_\rho,\mu_\rho}{\minimize} & \frac{1}{N}\sum^N_{i=1}\left\|f(\bar{x}+w_i) - \mathbf{J}_\rho w_i - \mu_\rho \right\|^2_2. \\
\end{aligned}
\end{equation}
Unlike the first-order estimator, the zeroth-order estimator is still unbiased for discontinuous functions \cite{suh2022differentiable}. However, the variance of this estimator is often higher than that of the first-order estimator used in Lemma \ref{lem:reparametrization}. This effect is also illustrated in Example \ref{ex:computation}.

\begin{example}\normalfont \label{ex:computation}(\textbf{ReLU, Heaviside, Delta}).
    To better illustrate different computational methods and their properties, we revisit an example from Table \ref{table:examples}, where $f$ is the ReLU function and $\rho$ is a logistic distribution, for which the smooth surrogagte $f_\rho$ is the Softplus. 
    
    We illustrate the three different methods, introduced in Sec.\ref{sec:smoothing_computation_schemes}, to compute the gradient of $f_\rho$. In the analytic scheme, we simply differentiate softplus to obtain a sigmoid function. The first-order randomized version averages the gradients $\partial f /\partial x$ (a Heaviside) with the samples drawn from $\rho$. Finally, the zeroth-order version uses Stein's lemma for the logistic distribution, which has the score function $S(w)=\tanh (w)$. The results are plotted in Fig.\ref{fig:computation}. Note that both quantities converge to the analytic gradient on average, with the zeroth-order gradient having more variance than the first-order one.
    
    Lastly, given $g\coloneqq\partial f /\partial x$ as the Heaviside, the smooth surrogate $g_\rho$ becomes the Sigmoid. We illustrate the three methods of computing the gradient of $g_\rho$, which is equivalent to $d^2f_\rho/dx^2$. Using the analytic scheme, we recover the density function for the logistic distribution, and the zeroth-order estimator is also successful in recovering the gradient. However, as the gradients of the Heaviside (i.e. dirac-delta) is almost-surely zero, the first-order estimator becomes biased and fails to estimate the gradient of $g_\rho$. 
    \begin{figure}[thpb]
	\centering\includegraphics[width = 0.5\textwidth]{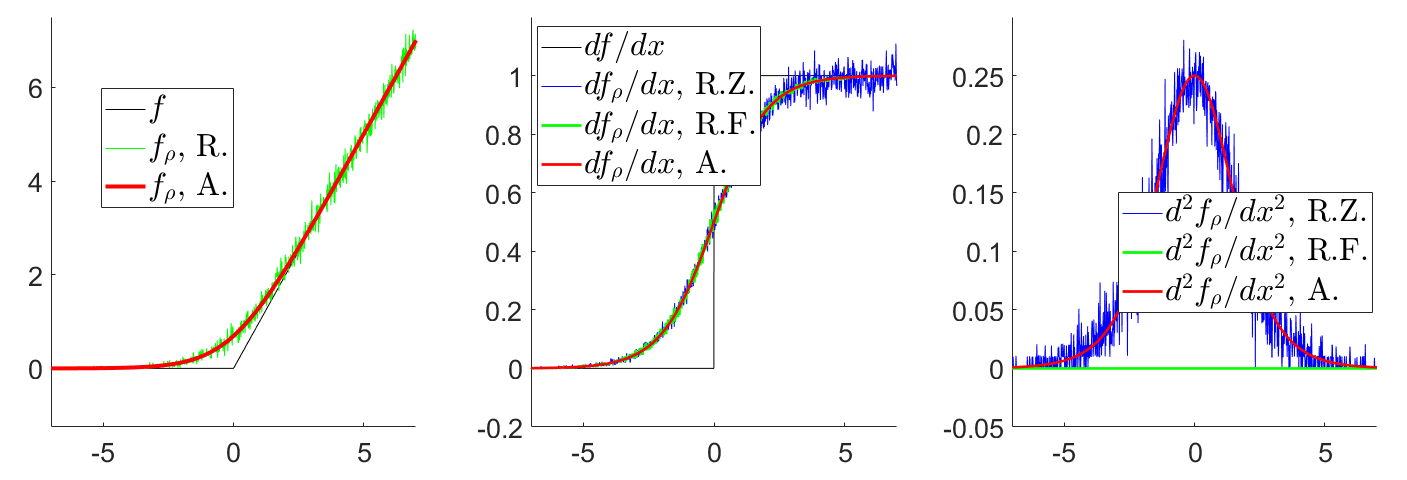}
    \caption{A. analytic, R.F. randomized first, R.Z. randomized zero. Left: ReLU $f$ (black), its smooth surrogate $f_\rho$, the softplus (red), and the Monte-Carlo approximation of $f_\rho$ (green). Middle: Gradients of $f_\rho$ computed using different methods. Right: Gradient of $g_\rho$, which we denote as $d^2f_\rho/dx^2$, computed by different computation methods. Color scheme is identical.}
	\label{fig:computation}
    \end{figure}
\end{example}

Finally, we note that in the context of randomized smoothing, the first-order and the zeroth-order gradient estimators can be combined and interpolated, leading to a $\alpha$-ordered gradient estimator where $\alpha\in[0,1]$ \cite{suh2022differentiable}. Such interpolation smoothly trades off the bias and variance characteristics of the first and zeroth-order gradients. 

\subsection{Smoothing of Dynamical Systems}\label{sec:smoothdynamics}
Our theorems in the setting of functions readily extends to dynamical systems. We overload the notation on the variables for convention, and say the system has state $x\in\mathbb{R}^n$, input $u\in\mathbb{R}^m$, and map $f:\mathbb{R}^n\times\mathbb{R}^m\rightarrow\mathbb{R}^n$, 
\begin{equation}
\label{eq:f_x_u}
x_+ = f(x, u).
\end{equation}
where we use $(\mathbf{A}(\bar{x},\bar{u}),\mathbf{B}(\bar{x},\bar{u}),c(\bar{x},\bar{u}))$ to parameterize the linear system $\bar{f}$ that best describes $f$ around some nominal point $(\bar{x},\bar{u})$, 
\begin{equation}
    \begin{aligned}
    \bar{f}(x,u) = \mathbf{A}(\bar{x},\bar{u})\underbrace{(x-\bar{x})}_{=:\delta x } + \mathbf{B}(\bar{x},\bar{u})\underbrace{(u-\bar{u})}_{=:\delta u} + c(\bar{x},\bar{u}).
    \end{aligned}
\end{equation}
For dynamical systems, the Taylor expansion requires the system Jacobian, 
\begin{equation}\label{eq:linearization}
    \begin{aligned}
    \mathbf{A}(\bar{x},\bar{u}) & = \frac{\partial}{\partial x}f(x,u)|_{x=\bar{x},u=\bar{u}} \\
    \mathbf{B}(\bar{x},\bar{u}) & =\frac{\partial}{\partial u}f(x,u)|_{x=\bar{x},u=\bar{u}} \\
    c(\bar{x},\bar{u}) & = f(\bar{x},\bar{u}).
    \end{aligned}
\end{equation}
Similarly in this setting, the smooth surrogate can be defined as the surrogate dynamics \begin{equation}
    f_\rho(x,u) = \mathbb{E}_{w\sim\rho}[f(x+w^x,u+w^u)],
\end{equation}
where $w^x$ is the component of $w$ that corresponds to $x$, and $w^u$ is defined similarly. The model parameters of the locally linear model $\bar{f}_\rho$ are given as follows:
\begin{equation}
\label{eq:ABc_rho}
    \begin{aligned}
    \mathbf{A}_\rho(\bar{x},\bar{u}) & =  \frac{\partial}{\partial x}\mathbb{E}_{w\sim\rho}[f(x+w^x,u+w^u)]|_{x=\bar{x},u=\bar{u}} \\
    \mathbf{B}_\rho(\bar{x},\bar{u}) & =  \frac{\partial}{\partial u}\mathbb{E}_{w\sim\rho}[f(x+w^x,u+w^u)]|_{x=\bar{x},u=\bar{u}} \\
    c_\rho(\bar{x},\bar{u}) & = \mathbb{E}_{w\sim\rho}[f(\bar{x}+w^x,\bar{u}+w^u)].
    \end{aligned}
\end{equation}

The previous theorems (Lemma \ref{lem:reparametrization}, Lemma \ref{lem:stein}, Theorem \ref{thm:regression}) also extend to this setting with slight changes in notation. In the remaining sections, we shorthand the notation and refer to the model parameters as a matrix instead of a function for compactness (e.g. $\mathbf{A}_\rho$ instead of $\mathbf{A}_\rho(\bar{x},\bar{u})$).

\subsubsection{Analytic Smoothing}
If $f$ is sufficiently structured, we can analytically obtain $f_\rho$ as well as the parameters for its linearization $(\mathbf{A}_\rho,\mathbf{B}_\rho,c_\rho)$ without a sampling procedure.

\subsubsection{First-Order Randomized Smoothing}
We use the following estimators for first-order randomized smoothing of dynamical systems,
\begin{subequations}
\label{eq:montecarlo_ABc}
\begin{align}
\mathbf{A}_\rho & \approx \textstyle \frac{1}{N}\sum^N_{i=1} \left[\textstyle\frac{\partial}{\partial x}f(\bar{x}+w_i^x,\bar{u}+w_i^u\right] \quad w_i\sim \rho \label{eq:montecarlo_ABc:A} \\ 
\mathbf{B}_\rho & \approx \textstyle \frac{1}{N}\sum^N_{i=1} \left[\textstyle\frac{\partial}{\partial u}f(\bar{x}+w_i^x,\bar{u}+w_i^u)]\right] \quad w_i\sim \rho \\ 
c_\rho & \approx \textstyle \frac{1}{N}\sum^N_{i=1} \left[\textstyle f(\bar{x}+w_i^x,\bar{u}+w_i^u)]\right] \quad w_i\sim \rho \label{eq:montecarlo_ABc:c} 
\end{align}
\end{subequations}
where $\partial f/\partial x$, and $\partial f/\partial u$ are Jacobians of the dynamics.

\subsubsection{Zeroth-Order Randomized Smoothing}
Similar to \eqref{eq:generic_f_local_least_squares}, we can estimate the gradients using least-squares when $\rho$ is Gaussian:
\begin{equation}
    \label{eq:montecarlo_ABc_zero}
    \begin{aligned}
    \mathbf{A}_\rho,\mathbf{B}_\rho &  = \underset{\mathbf{A,B}} {\text{argmin}}\textstyle\sum^N_{i=1}\\
    & \| f(\bar{x}+w_i^x,\bar{u}+w_i^u) - \mathbf{A} w_i^x - \mathbf{B} w_i^u - c_\rho\|^2_2
    \end{aligned}
\end{equation}
where $c_\rho$ is computed with \eqref{eq:montecarlo_ABc:c}.

\section{Convex Quasi-dynamic Differentiable Contact Dynamics}\label{sec:quasidynamic}
\noindent In this section, we present the Convex, Quasi-dynamic, Differentiable Contact (CQDC) model which is used for the contact-rich planning tasks in the rest of this paper. We start by elaborating on the advantages of planning with quasi-dynamic models. We then detail the convex\footnote{
By calling the forward dynamics $x_+ = f(x, u)$ convex, we mean that a convex optimization program is solved when computing $x_+$ from a given $(x, u)$ pair. It does not mean that $f(x, u)$ is a convex function in $(x, u)$.  
} forward dynamics formulation and how to take its derivatives, and conclude with remarks on implementation. 

\subsection{Why Plan with Quasi-dynamic Models?} \label{sec:why_quasidynamic}
Quasi-dynamic models, which previous works have used extensively for robotic manipulation \cite{mason2001mechanics, motioncones, aceituno2020global, cheng2021contact, pang2021easing}, simplify Newtonian dynamics by removing terms related to velocity and acceleration, and focusing on contact forces at the core of robot-object interactions. Although the model cannot describe highly-dynamics behaviors such as spinning a pen between fingers \cite{mordatch2012contact}, it holds up for a wide variety of manipulation tasks, including many that involve dexterous hands \cite{andrychowicz2020learning}. 

Halving the number of states by removing velocity is one of the most recognized advantages of quasi-dynamic systems \cite[\textsection 10.1]{mason2001mechanics}.
However, by focusing on transitions between static equilibria and ignoring transients induced by damping and acceleration, a quasi-dynamic model can predict further into the future while taking fewer steps than its second-order counterpart, thereby reducing the number of steps it requires to reach the goal.
We believe this temporal aspect of quasi-dynamic models' advantage is important but often overlooked in the literature. In practice, many trajectory optimization formulations for manipulation enforce the second-order Newtonian dynamics \cite{landry2019differentiable, kurtz2022contact}. Although modeling the transients allows the discovery of more dynamic behaviors, we believe, especially in the manipulation setting, that the added computational complexity due to taking many steps for long planning horizons oftentimes outweighs the benefits.

\subsection{Forward Dynamics \label{sec:convex_quasi_dynamic_contact_dynamics}}
Despite the numerical advantage of convex contact dynamics formulations \cite{anitescu2006optimization, todorov2012mujoco, castro2021unconstrained}, some formulations can produce highly inaccurate contact forces with poorly-chosen hyperparameters \cite{kolbert2016experimental}. In this work, we adopt Anitescu's convex relaxation of the Coulomb friction constraints \cite{anitescu2006optimization}. Anitescu's convex relaxation is equivalent to common LCP formulations \cite{stewart2000rigid} in non-sliding contacts or in separation, and introduces a non-physical yet mild ``boundary layer'' effect between relatively-sliding objects \cite{pang2021convex, castro2021unconstrained}. Not only does Anitescu's convex friction model have the simulation step size $h$ as the only hyperparameter, the ``boundary layer'' also disappears as $h \rightarrow 0$. The step size therefore can serve as a knob that allows the planner to trade the dynamics model's predictive power for physical accuracy.

In the robotic manipulation setting, a system can be divided into $\nA$ \emph{actuated} Degree of Freedoms (DOFs), which correspond to the robots, and $\nU$ \emph{unactuated} DOFs, which correspond to the objects. The configurations of objects and robots are denoted respectively by $\qu \in \R[\nU]$ and $\qa \in \R[\nA]$. For quasi-dynamic systems, the system state is defined by $x \coloneqq q \coloneqq (\qu, \qa)$. Similar to \cite{pang2021convex}, we model the robots as impedances \cite{hogan1985impedance}, which in the quasi-dynamic setting reduces to springs with a diagonal stiffness matrix $\Ka \in \R[\nA \times \nA]$. Accordingly, the input $u \in \R[\nA]$ is defined by the commanded positions of the robots' joints, which can also be interpreted as the equilibrium positions of the springs.

The discretized quasi-dynamic equations of motion are
\begin{subequations}
\label{eq:q_dynamics_eom}
\begin{align} 
h \Ka \left(\qa + \dqa - u \right) &= h\tauA + \sum_{i=1}^{\nC} \left( \Ja[i](q) \right)^\intercal \lambda_i, \label{eq:q_dynamics_eom:a}\\
\left( \frac{\epsilon}{h} \Mu (q) \right) \dqu &= h\tauU + \sum_{i=1}^{\nC} \left( \Ju[i](q) \right)^\intercal \lambda_i, \label{eq:q_dynamics_eom:u}
\end{align}
\end{subequations}
where $h \in \R_{++}$ is the step size in seconds; $\Mu(q) \in \R[\nU \times \nU]$ is the configuration-dependent mass matrix of the objects; $\epsilon \in \R_{+}$ is a small regularization constant; 
$\tauA \in \R[\nA]$ and $\tauU \in \R[\nU]$ are non-contact external torques (e.g. due to gravity) for robots and objects;
the change in system configuration from the current to the next step is $\dq \coloneqq (\dqu, \dqa)$.
For 3D systems, it is more common to work with angular velocity instead of rotational displacement. As our discussion trivially extends to angular velocity by using a linear map (e.g. \cite[Section II]{castro2021unconstrained}), we use displacement in this paper for brevity.

There are $\nC$ contact pairs at the current time step. For the $i$-th contact, $\lambda_i \in \R[3]$ is the contact impulse; $\Ja[i](q) \in \mathbb{R}^{3 \times \nA}$ and $\Ju[i](q) \in \mathbb{R}^{3 \times \nU}$ are the contact Jacobians \cite{anitescu1996formulating} of the robots and the objects, respectively. The contact Jacobians are functions of $q$, as they depend on the contact points and normals at $q$ computed by a collision detector. For brevity, we will drop the explicit dependency on $q$ from here onward.

Equation \eqref{eq:q_dynamics_eom:a} states that the robot joint positions' deviation from their commanded values, as a result of external impulses, is proportional to the stiffness $\Ka$. Equation \eqref{eq:q_dynamics_eom:u} can have different physical interpretations depending on the value of the regularization constant $\epsilon$.
If $\epsilon = 0$, \eqref{eq:q_dynamics_eom:u} is the force (impulse) balance of the objects. This corresponds to the exact quasi-static formulation in  \cite{pang2021convex}. 
If $\epsilon = 1$, \eqref{eq:q_dynamics_eom:u} corresponds to Mason's classical definition of quasi-dynamic systems \cite{mason2001mechanics}.
Note that for any $\epsilon$, the momentum gained at every step due to external impulses is discarded at the next time step, which is characteristic of the highly damped behavior typical in quasi-static systems. In our physical experiments in later sections where objects are light, we find that the dynamics is closer to being exactly quasi-static ($\epsilon = 0$) than being classically quasi-dynamic ($\epsilon = 1$). Therefore, we pick $\epsilon$ to be as small as possible without causing numerical issues.

In addition, the Coulomb friction model requires the contact impulses $\lambda_i$ to stay inside the friction cone, and the relative sliding velocities to satisfy the maximum dissipation principle \cite{stewart2000rigid}. To enforce these constraints, we first introduce some additional notation. The contact Jacobian for contact $i$ can be written as
\begin{equation}
\label{eq:contact_jacobian_i}
\J_i \coloneqq [\Ju[i], \Ja[i]] \coloneqq 
\begin{bmatrix}
\Jn[i] \\
\Jt[i]
\end{bmatrix}
\in \R[3 \times (\nU + \nA)],
\end{equation}
where $\Jn[i] \in \R[1 \times (\nU + \nA)]$ maps the generalized velocity of the system to the normal contact velocity, and $\Jt[i] \in \R[2 \times (\nU + \nA)]$ to the tangent velocities. Next, the friction cone at contact $i$, $\mathcal{K}_i$, and its dual cone $\mathcal{K}_i^\star$, are denoted by
\begin{subequations}
\label{eq:contact_cones}
\begin{align}
\mathcal{K}_i &\coloneqq \left\{\lambda_i = (\lambda_{\mathrm{n}_i}, \lambda_{\mathrm{t}_i}) \in \R[3] | \mu_i \lambda_{\mathrm{n}_i} \geq \sqrt{\lambda_{\mathrm{t}_i}^\intercal \lambda_{\mathrm{t}_i}} \right\}, \label{eq:contact_cones:lambda}\\
\mathcal{K}_i^\star &\coloneqq \left\{v_i = (v_{\mathrm{n}_i}, v_{\mathrm{t}_i}) \in \R[3] | v_{\mathrm{n}_i} \geq \mu_i \sqrt{v_{\mathrm{t}_i}^\intercal v_{\mathrm{t}_i}} \right\}, \label{eq:contact_cones:v}
\end{align}
\end{subequations}
where $\mu_i$ is the friction coefficient; the dual variable $v_i$ can also be interpreted as the relative contact velocity for contact $i$; the subscripts $(\cdot)_\mathrm{n}$ and $(\cdot)_\mathrm{t}$ indicate respectively the normal and tangential components. 

With this notation, Anitescu's frictional contact constraints can be written as:
\begin{subequations}
\label{eq:friction_constraints}
\begin{align}
v_i \coloneqq 
\J_i
{\delta q}
+
\begin{bmatrix}
\phi_i \\
0_2
\end{bmatrix}
&\in \mathcal{K}_i^\star,\\
\lambda_i &\in \mathcal{K}_i, \\
v_i^\intercal \lambda_i &= 0,
\end{align}
\end{subequations}
where $\phi_i \in \R$ is the signed distance for contact $i$ at the current time step. These constraints enforce the Coulomb friction model exactly when a contact is sticking (not sliding) and in separation. In sliding, these constraints enforce maximum dissipation, but adds a small gap between the two relatively-sliding objects. Anitescu showed that this gap converges to 0 as $h \rightarrow 0$ \cite{anitescu2006optimization}.

Remarkably, this implies that the quasi-dynamic equations of motion \eqref{eq:q_dynamics_eom}, together with the friction constraints \eqref{eq:friction_constraints}, are the KKT optimality conditions \cite[\textsection 5.9]{boyd2004convex} of the following convex Second-Order Cone Program (SOCP) \cite{anitescu2006optimization}:
\begin{subequations}
\label{eq:q_dynamic_socp}
\begin{align}
&\underset{\dq}{\minimize} \; \frac{1}{2} \dq^\intercal \mathbf{Q} \dq + b^\intercal \dq, \; \text{subject to} \label{eq:q_dynamic_socp:cost}\\
&\qquad \qquad \J_i
{\delta q}
+
\begin{bmatrix}
\phi_i \\
0_2
\end{bmatrix}
\in \mathcal{K}_i^\star, \; \forall i \in \{1 \dots \nC\}, \; \text{where} \label{eq:q_dynamic_socp:constraint}\\
&\mathbf{Q} \coloneqq \begin{bmatrix} \epsilon\Mu/h & 0 \\ 0 & h \Ka \end{bmatrix}, \;
b \coloneqq - h\begin{bmatrix} \tauU \\ \Ka(u - \qa) + \tauA \end{bmatrix},
\end{align}
\end{subequations}
whose primal and dual solutions, $\dq^\star$ and $\lambda^\star \coloneqq (\lambda_1^\star, \dots, \lambda_{\nC}^\star)$, can be obtained by conic solvers such as \cite{mosek, scs}. Note that without regularization ($\epsilon=0$), the objective \eqref{eq:q_dynamic_socp:cost} would be positive semi-definite, and thus \eqref{eq:q_dynamic_socp} could have multiple solutions.

The SOCP \eqref{eq:q_dynamic_socp} reduces to a Quadratic Program (QP) when there is no friction or if the friction cone can be described by linear constraints. The second-order friction cones \eqref{eq:contact_cones} can be equivalently represented as linear constraints in the planar case, or be approximated with polyhedral cones \cite{stewart2000rigid} in the 3D case. However, this approximation introduces non-physical anisotropic behaviors \cite{li2018implicit, howell2022dojo}, which is why \eqref{eq:contact_cones} is preferred.

\subsection{Derivatives of Forward Dynamics}\label{sec:quasi_dynamic_derivatives}
We illustrate how to compute the derivatives of the system configuration at the next step, $q_+$, with respect to the current $q$ and $u$. Let us express the forward contact dynamics in the standard dynamical system form \eqref{eq:f_x_u}:
\begin{equation}
\label{eq:f_q_u}
q_+ = f(q, u) = q + \dq^\star(q, u),
\end{equation}
where $\dq^\star$ is the solution to  \eqref{eq:q_dynamic_socp}. Taking the derivatives of \eqref{eq:f_q_u} yields
\begin{equation}
\label{eq:q_dynamics_AB}
\A = \DfDx{f}{q} = \I + \DfDx{\dq^\star}{q}, \;
\B = \DfDx{f}{u} = \DfDx{\dq^\star}{u},
\end{equation}
where $\DfDxLine{\dq^\star}{q}$ and $\DfDxLine{\dq^\star}{u}$ can be expanded using the chain rule into:
\begin{subequations}
\label{eq:dq_star_q_u}
\begin{align}
\DfDx{\dq^\star}{q} &= \DfDx{\dq^\star}{b} \DfDx{b}{q} + \DfDx{\dq^\star}{\mathbf{Q}}\DfDx{\mathbf{Q}}{q} + \sum_{i=1}^{\nC} \DfDx{\dq^\star}{\J_i}\DfDx{\J_i}{q} + \DfDx{\dq^\star}{\phi_i}\DfDx{\phi_i}{q}, \label{eq:DdqDq}\\
\DfDx{\dq^\star}{u} &= \DfDx{\dq^\star}{b} \DfDx{b}{u}. \label{eq:DdqDu}
\end{align}
\end{subequations}

Similar to other differentiable simulators based on implicit time-stepping \cite{werling2021fast, howell2022dojo}, we compute the derivatives of the solution $\dq^\star$ with respect to the problem data $(\mathbf{Q}, b, \J_i, \phi_i)$ by applying the implicit function theorem to the KKT optimality conditions of the convex program \eqref{eq:q_dynamic_socp} \cite{agrawal2019differentiatingcone}. 
Then, the derivatives of $(\mathbf{Q}, b, \J_i, \phi_i)$ with respect to $q$ and $u$ can be straightforwardly computed using either automatic differentiation or a more specialized method that takes advantage of the structure of rigid body systems \cite{carpentier2018analytical}.

As the derivatives \eqref{eq:dq_star_q_u} are discontinuous functions of $(q, u)$ \cite{bundledgradients}, they are not a good local approximation of the CQDC dynamics unless multiple gradients are computed for different state-action pairs and averaged using the first-order randomized smoothing scheme (Sec. \ref{sec:randomized_smoothing}).

\subsection{Implementation}
As all existing rigid-body differentiable simulators, to the best of our knowledge, assume second-order Newtonian dynamics, it is difficult and inefficient to modify them to support our CQDC dynamics formulation. Therefore, we implemented the proposed dynamics formulation using the Drake robotics toolbox \cite{drake}. Although our implementation is not heavily optimized, it is adequate for computing contact-rich plans within a reasonable amount of wall-clock time (a few minutes). 

For the forward dynamics (Sec. \ref{sec:convex_quasi_dynamic_contact_dynamics}), we use Drake's \code{MultibodyPlant} and \code{SceneGraph} for collision detection and the computation of the object mass matrix $\Mu$, contact Jacobians $\J_i$ and signed distances $\phi_i$. The SOCP \eqref{eq:q_dynamic_socp} can then be constructed with \code{MathematicalProgram} and solved with a third-party solver of our choice.

For the dynamics derivatives (Sec.\ref{sec:quasi_dynamic_derivatives}), we have a custom implementation for differentiating through the KKT optimality conditions of an SOCP using Eigen's \cite{eigenweb} linear solvers, allowing us to efficiently compute the partial derivatives of $\dq^\star$ with respect to $(\mathbf{Q}, b, \J_i, \phi_i)$ from primal-dual solutions $(\dq^\star, \lambda^\star)$ given by third-party conic solvers. As for the partials of $(\mathbf{Q}, b, \J_i, \phi_i)$ with respect to $q$, we note that $b$ is a linear function of $q$;  $\DfDxLine{\phi_i}{q} = \Jn[i]$; $\DfDxLine{\J_i}{q}$ and $\DfDxLine{\mathbf{Q}}{q}$ are computed with Drake's forward-mode automatic differentiation.

Finally, in order to avoid discontinuities coming from collision detection, we curate our system models so that every contact pair is either sphere-sphere, sphere-box, or sphere-cylinder, which means the contact points and normals change smoothly with the system configuration $q$ \cite{SiggraphContact22}. For example, both the box-shaped fingers of the Allegro hand in Fig. \ref{fig:banner}a-d and the box-shaped manipuland in Fig. \ref{fig:banner}b  are represented as arrays of inscribing spheres for the purpose of collision detection.
We note that this limitation can be alleviated with smoothing over collision geometries \cite{simonsinglelevel,randomizedsmoothingcollision}.

\section{Smoothing of Contact Dynamics}\label{sec:smoothingequivalence}

\noindent In this section, we combine our formalism and the computation methods of smoothing in Sec.\ref{sec:localsmoothing} and the CQDC dynamics in Sec.\ref{sec:quasidynamic}, resulting in a class of methods that efficiently make local approximations to the smooth surrogate of our convex quasi-dynamic contact model. Throughout this section, we assume we compute a local model of the smooth surrogate around some nominal state-input pair $(\bar{q},\bar{u})$. 

\subsection{Randomized Smoothing of Contact Dynamics}
We follow the method in Sec.\ref{sec:smoothdynamics} in order to perform randomized smoothing of the contact dynamics. For first-order randomized smoothing, we utilize \eqref{eq:montecarlo_ABc} with access to gradients of the contact dynamics obtained in Sec.\ref{sec:quasi_dynamic_derivatives}. We similarly do zero-order randomized smoothing using \eqref{eq:montecarlo_ABc_zero}.

However, there is a caveat in the randomized smoothing scheme: if the sampling distribution $\rho$ has infinite support, the sampled state $\bar{q} + w^q_i$ could violate the non-penetration constraint for rigid-bodies, i.e. $\phi(\bar{q} + w^q_i) < 0$ for some $i$.

Although reasoning about the dynamics $f(q, u)$ with an infeasible (penetrating) $q$ may seem ill-posed \cite{contactkalmanfilter}, we can \emph{define} the dynamics from an infeasible $q$ as the projection of $q$ to the ``nearest" point in the feasible (non-penetrating) set. The notion of nearest can be defined in terms of the work required to move the system configuration by $\dq$, which is precisely the quadratic cost \eqref{eq:q_dynamic_socp:cost} (divided by the step size $h$). This projection problem can then be written as
\begin{subequations}
\label{eq:projection_as_optimization}
\begin{align}
\underset{\dq}{\minimize} \; &\frac{1}{2} \dq^\intercal \mathbf{Q} \dq + b^\intercal \dq, \; \text{subject to} \\
& \phi_i(q + \delta q) \geq 0, \; i \in \{1 \dots \nC\}, \label{eq:projection_as_optimization:constraint}
\end{align}
\end{subequations}
where \eqref{eq:projection_as_optimization:constraint} is the non-linear non-penetration constraint. 

While the projection in \eqref{eq:projection_as_optimization} is difficult to solve in general, we can linearize the constraint \eqref{eq:projection_as_optimization:constraint} in order to locally approximate the problem as a QP. When the constraint is linearized, the problem remarkably becomes equivalent to the frictionless special case of the CQDC dynamics \eqref{eq:q_dynamic_socp}. In other words, projection is simply another interpretation of what the CQDC dynamics does within the penetrating regime, and no other explicit treatment is required for projection other than the evaluation of CQDC dynamics. In practice, due to the local nature of linearization, we use a low-variance distribution $\rho$ to sample states, though higher variance can be used for inputs.

When samples within the penetrating regime are projected onto the boundary of the feasible set and then averaged, the expected value of such a distribution creates a \emph{stochastic force field} that pushes the object away from feasible set's boundary. We illustrate this phenomenon through a simple example. 

\begin{example}\label{ex:projection}\normalfont\textbf{(The Stochastic Force Field)} Consider the dynamics of an unactuated 1D block with a wall occupying $q \leq 0$ (Fig. \ref{fig:projection}a), such that the physical dynamics is identity \emph{if} the block is in a non-penetrating configuration, $f(q)=q$ if $q\geq 0$. The dynamics within the penetrating regime is not well-defined physically; yet, applying the quasi-dynamic equations of motions \eqref{eq:projection_as_optimization} to this system gives 
\begin{subequations}
\label{eq:projection}
\begin{align}
\underset{\delta q}{\minimize} \; &\frac{1}{2} m (\delta q)^2, \; \text{subject to} \\
& q + \delta q \geq 0.
\end{align}
\end{subequations}
which has the following solution,
\begin{equation}
\label{eq:1d_projection_solution}
f(q) = q + \delta q =  \begin{cases}
    q & \text{ if } q \geq 0  \text{ (no penetration) }\\
    0 & \text{ else } \text{ (penetration) }\\
\end{cases}
\end{equation}

One interpretation of \eqref{eq:1d_projection_solution} is that configurations inside the wall gets projected onto the wall. By taking an average according to such dynamics, the expectation pushes the box away from the wall as illustrated in Fig. \ref{fig:projection}, which creates a stochastic force field. Note that with this interpretation, the gradients are also well-defined within the penetrating regime - an infinitesimal change of position in the penetrating configuration does not have any effect on the location of projection in this example, thus $\partial f(q)/\partial q=0$ if $q < 0$ (But note that after smoothing, $\partial f_\rho(q)/\partial q > 0$). For more complex geometries, the location of projection changes due to changes in the surface, which we connect back to the presence of the $\partial \mathbf{J}/\partial q$ and $\partial \phi/\partial q$ terms in \eqref{eq:DdqDq}.

\begin{figure}[t]
\centering\includegraphics[width = 0.48\textwidth]{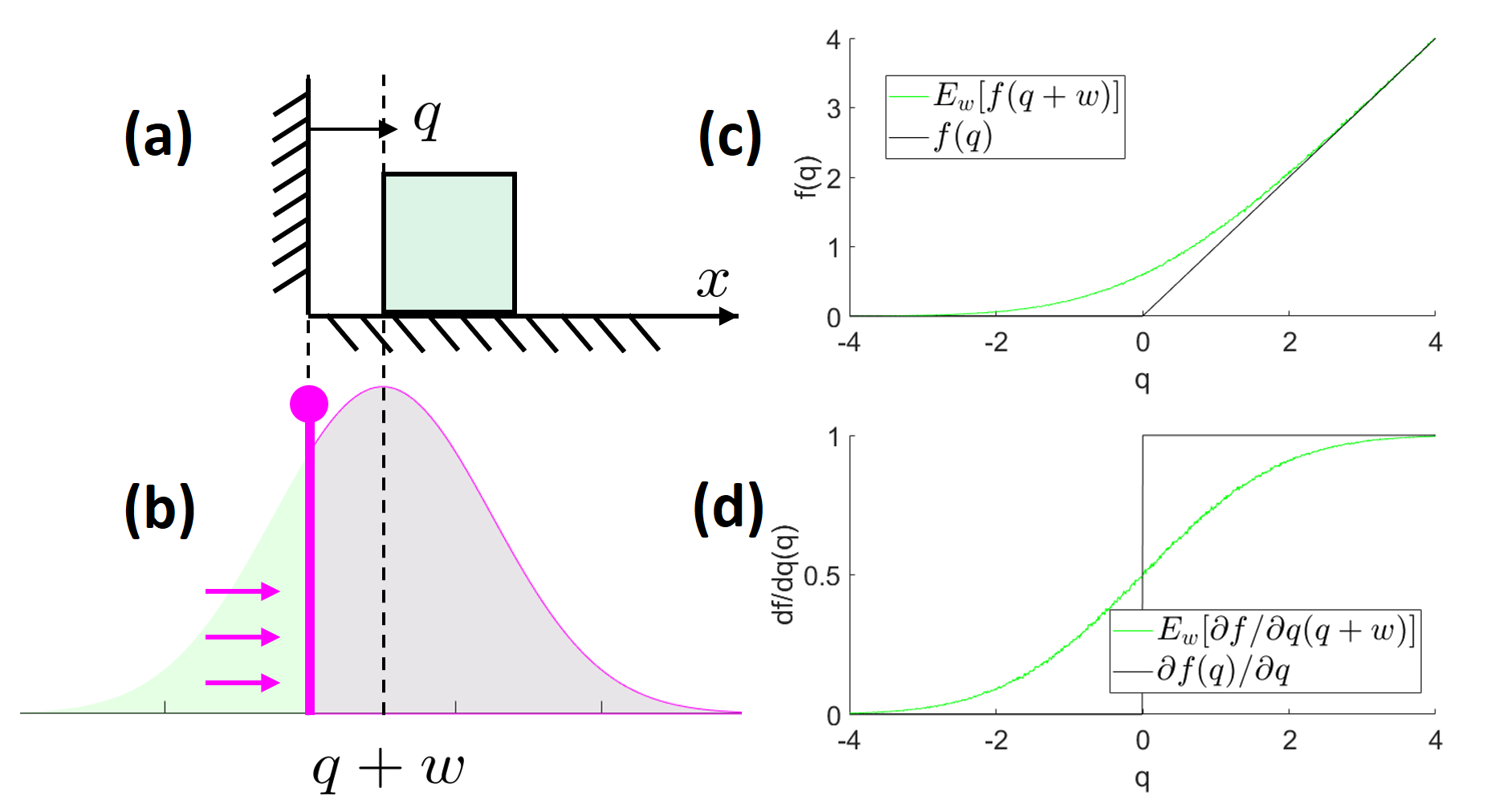}
\caption{Figure for Example \ref{ex:projection}, where quasi-dynamics of motion is interpreted as a projection operator. (\textbf{a}) Illustration of the system. (\textbf{b}) Distribution of $q + w$ (green) and $f(q + w)$ (pink). Note that the samples for which $q+w_i<0$ have been projected onto the surface into a delta function, and the expectation of the pink distribution lies on the right side of $q$, creating a stochastic force field effect. (\textbf{c}) CQDC dynamics and its randomized smoothed version. (\textbf{d}) Gradients of the CQDC dynamics obtained with first-order randomized smoothing.}
\label{fig:projection}
\vskip -0.2 true in
\end{figure}
\end{example}

\subsection{Analytic Smoothing of Contact Dynamics}\label{sec:analyticsmoothing}
Randomized smoothing simply involves solving \eqref{eq:q_dynamic_socp} repeatedly and averaging the resulting $q_+$, $\A$, and $\B$. In contrast, analytic smoothing (Sec.\ref{sec:analytic_smoothing}) of the CQDC dynamics \eqref{eq:q_dynamic_socp} is not as straightforward, since the solution to an optimization problem does not give easy access to explicit forms from which we can analytically design smooth surrogates.

To perform analytic smoothing,  we can convert the hard constraints in a convex program into costs using a penalty method. Specifically, we convert the constrained program \eqref{eq:q_dynamic_socp} into an unconstrained convex program using the log-barrier function, a common technique in the interior-point method for convex conic programs, with weight $\kappa$:
\begin{equation}
\label{eq:q_dynamics_log}
\begin{aligned}
\underset{\dq}{\minimize} \; &\frac{1}{2} \dq^\intercal \mathbf{Q} \dq + b^\intercal \dq \\
&- \frac{1}{\kappa} \sum_{i=1}^{\nC} \log \left[\frac{(\Jn[i] \dq + \phi_i)^2}{\mu_i^2} - (\Jt[i]\dq)^\intercal \Jt[i]\dq \right],
\end{aligned}
\end{equation}
whose solution converges to the solution of \eqref{eq:q_dynamic_socp} as $\kappa \rightarrow \infty$ \cite[\textsection 11.3 and \textsection 11.6]{boyd2004convex}. 

The log-barrier term in \eqref{eq:q_dynamics_log} can also be interpreted as the potential of a force field whose strength is inversely proportional to the distance to the boundary of the constraint \cite[p.567]{boyd2004convex}. For moderate values of $\kappa$, constraints can exert forces even though they are not active, achieving a smoothing effect similar to the ``force-at-a-distance" relaxation of complementarity constraints, which are commonly used in planning through contact methods such as \cite{posa2014direct, howell2022dojo}. We will further illustrate this similarity in Example \ref{ex:equivalence}.

From $\dq^\star$, the solution to the smoothed dynamics \eqref{eq:q_dynamics_log}, we can directly compute the smoothed gradient $\A_\rho$ and $\B_\rho$ in the same way as \eqref{eq:q_dynamics_AB}. Once again, the derivatives of $\dq^\star$ with respect to $q$ and $u$ are computed by applying the implicit function theorem to the optimality condition of \eqref{eq:q_dynamics_log}, which only consists of the stationarity condition due to the absence of conic constraints.

%Finally, we remark that despite the convexity of the smoothed dynamics \eqref{eq:q_dynamics_log}, we could not find an easy way to formulate it as a convex program using the interfaces of popular optimization parsers such as \cite{drake, diamond2016cvxpy}.

To solve \eqref{eq:q_dynamics_log}, we implemented an in-house solver using Newton's method \cite[\textsection 9.5]{boyd2004convex}, and find that our out-of-the-textbook implementation works robustly and reliably for all numerical experiments in Sec. \ref{sec:traj_opt} and \ref{sec:rrt_results}.

\subsection{Equivalence of Smoothing Schemes}\label{sec:smoothingequivalence}
In Sec.\ref{sec:localsmoothing}, we saw that randomized and analytic smoothing can be interpreted as different methods of computing the same quantity. Here, we show with examples that (\textbf{i}) for simple systems, we can derive the sampling distribution $\rho$ needed to obtain the smoothing given by the log-barrier-based analytic smoothing scheme; (\textbf{ii}) for more complex systems, randomized smoothing schemes using a Gaussian $\rho$ and the log-barrier-based analytic smoothing scheme result in qualitatively-similar smoothed dynamics.

\begin{figure}[t]
\centering
\includegraphics[width = 0.48\textwidth]{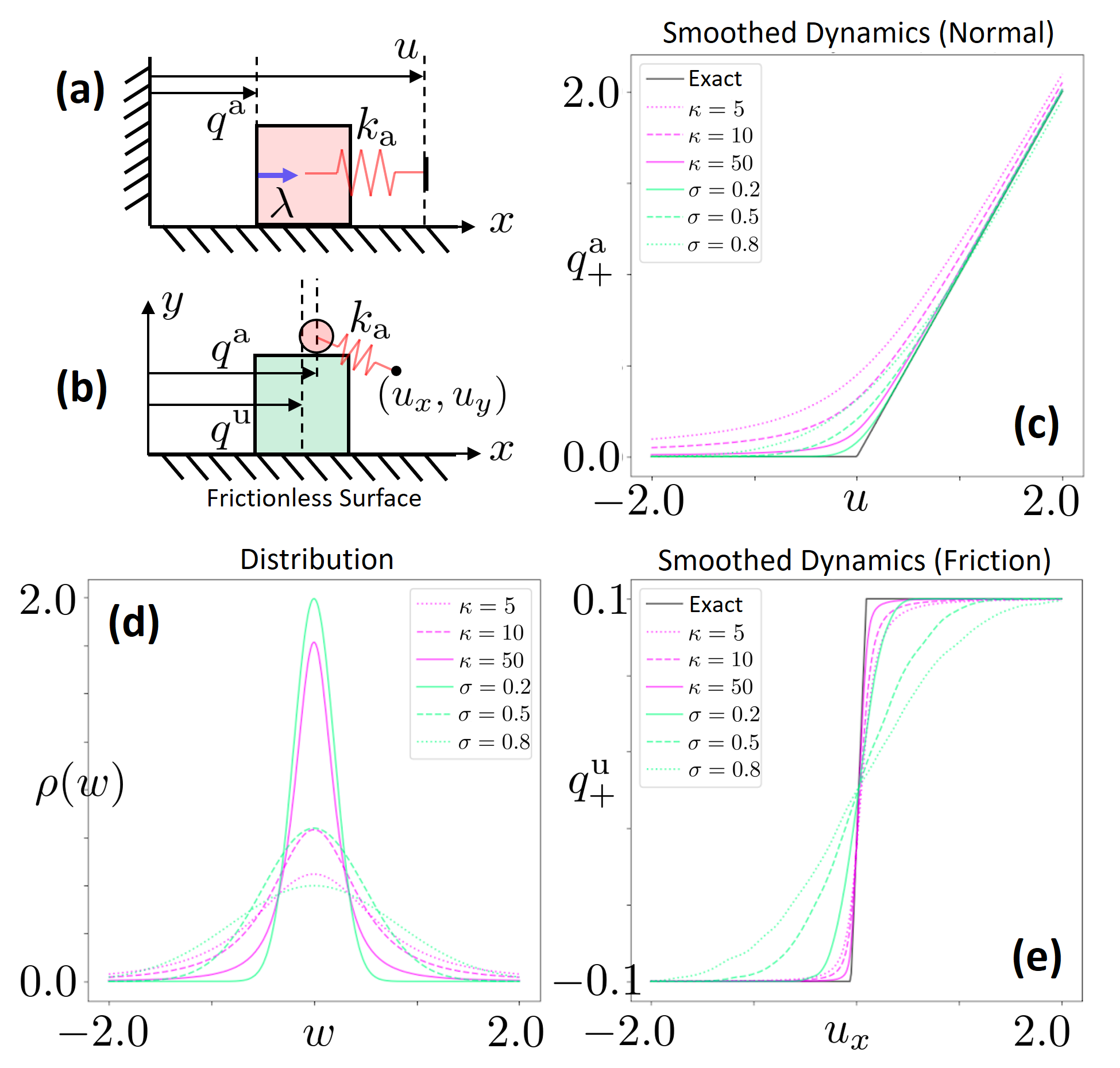}
\caption{
(\textbf{a}) A system consisting of an actuated cart constrained to slide on a frictionless surface, and a wall occupying $\qa \leq 0$. The actuator has stiffness $k_\textrm{a}$.
(\textbf{b}) A system consisting of an unactuated cart constrained to slide on a frictionless surface, and a ball actuated along both the $x$ and $y$ axes. The ball can touch the top surface of the cart with a frictional contact.
(\textbf{c}) Randomized and analytic smoothing of (a). Randomized smoothing, shown in green, is done with a Gaussian kernel with different variances $\sigma$. Analytic smoothing, shown in magenta, is done with different log-barrier weights $\kappa$. 
(\textbf{d}) Density functions of the Gaussian kernels (green) and the elliptical distributions used for analytic smoothing (magenta). 
(\textbf{e}) Randomized and analytic smoothing of (b). We plot $q^\mathrm{u}_+$ against $u_x$ for a fixed $u_y$ that is inside the cart. The linear region in the plot corresponds to sticking contact, and the flat regions to sliding.}
\label{fig:equivalence}
\vskip -0.2 true in
\end{figure}

\begin{example}\label{ex:equivalence} \normalfont(\textbf{Equivalence of Smoothing Schemes})
We start with the 1D frictionless system in Fig. \ref{fig:equivalence}a, whose dynamics is an instance of the frictionless CQDC dynamics \eqref{eq:q_dynamic_socp}:
\begin{equation}
\label{eq:q_dynamics_1d_cart}
\begin{aligned}
\underset{\dqa}{\minimize} \; &\frac{1}{2} hk_a (\dqa)^2 - h \ka (u - \qa) \dqa, \; \text{subject to} \\
&\qa + \dqa \geq 0,
\end{aligned}
\end{equation}

The KKT conditions of \eqref{eq:q_dynamics_1d_cart} are also the equations of motion of the system:
\begin{subequations}
\label{eq:wallcartdynamics}
\begin{align} 
h k_\mathrm{a} \left(\qa + \dqa - u\right) &= \lambda, \\
0 \leq \qa + \dqa  &\perp \lambda \geq 0,
\end{align}
\end{subequations}
which has the explicit solution
\begin{equation}
    q^\mathrm{a}_+ = \qa + \dqa =  \begin{cases}
        u & \text{ if } u \geq 0  \text{ (no contact) }\\
        0 & \text{ else } \text{ (contact) }\\
    \end{cases}.
\end{equation}

Also as an instance of \eqref{eq:q_dynamics_log}, \eqref{eq:q_dynamics_1d_cart} can be smoothed analytically by converting the constraint into a cost using the log-barrier function:
\begin{equation}
\underset{\dqa}{\minimize} \; \frac{1}{2} hk_a (\dqa)^2 - h \ka (u - \qa) \dqa - \frac{1}{\kappa} \log \left(\qa + \dqa \right),    
\end{equation}
whose optimality condition is obtained by setting the gradient of the smoothed cost to 0, yielding the equations of motion of the smoothed system: 
\begin{equation}
\label{eq:smooth_1d_cart}
hk_\mathrm{a}(\qa + \dqa - u) = \frac{1}{\kappa}\bigg(\frac{1}{\qa + \dqa}\bigg).
\end{equation}

The right-hand side of \eqref{eq:smooth_1d_cart} can be interpreted as an impulse whose magnitude is inversely proportional to the distance to the wall. Calling this impulse $\beta$, we note that
\begin{equation}\label{eq:complementarityrelaxation}
    \beta (\qa + \dqa) = \kappa^{-1},
\end{equation}
which is analogous to the common bilinear relaxation to the complementarity constraint for contact \cite{posa2014direct,howell2022dojo}.

The solution to \eqref{eq:smooth_1d_cart} is given by 
\begin{equation}
    q^\mathrm{a}_+=\frac{1}{2}\left(u + \sqrt{u^2 + \frac{4}{\kappa h^2k^2_\mathrm{a}}}\right)
\end{equation}
which is equivalent to randomized smoothing of the original dynamics under the following elliptical distribution
\begin{equation}
    \rho(w) = \sqrt{\frac{4\sigma }{(w^\intercal\sigma w + 4)^3}},
\end{equation}
where $\sigma\coloneqq hk\kappa $. We note that as opposed to the Gaussian, this distribution has a heavy tail. 
\end{example}

For more complex examples involving friction, such as the system in Fig. \ref{fig:equivalence}b, obtaining an analytic expression for $\rho$ that corresponds to \eqref{eq:q_dynamics_log} is difficult. Instead, we numerically illustrate the performance of the two smoothing schemes in Fig. \ref{fig:equivalence}e. (See \cite{bundledgradients} for more details). We note that with a different relaxation of complementary slackness, Howell \textit{et al.} showed a similar trend for the smoothed dynamics as the relaxation is tightened \cite[Fig. 7]{howell2022dojo}.

For more general settings beyond our simple examples, we expect there to be subtle mathematical differences between analytic and randomized smoothing, though they demonstrate similar empirical behavior in our work.

\section{Case Study: Trajectory Optimization \label{sec:traj_opt}}
\noindent Throughout the previous sections, we have developed the theory and computational methods to obtain local linear models of the smooth surrogate of the proposed CQDC dynamics. These tools are not specific to any planning algorithm: they can improve the performance of most iterative planning algorithms that rely on gradients.

In this section, we demonstrate the efficacy of smoothed CQDC dynamics on trajectory optimization for systems with contacts.
Although smoothing of hard contact constraints has been widely utilized to improve convergence \cite{posa2014direct, howell2022dojo, howell2022trajectory}, existing methods still struggle with complex problems such as dexterous manipulation. In particular, the quality of solutions can be very sensitive to initial guesses \cite{onol2020tuning}.
In this context, we show that a variant of trajectory optimization, with the help of smoothed CQDC dynamics and only a trivial initial guess, can perform well even on dexterous manipulation tasks.

\subsection{Iterative MPC with Smoothing \label{sec:iMPC}}
The variant of trajectory optimization algorithm used in this section, which we call iterative MPC (iMPC), is an iLQR-inspired algorithm proposed in \cite{bundledgradients}. We briefly summarize iMPC here for completeness, and illustrate how smoothing can be easily integrated into iMPC.

Consider the problem of finding an optimal sequence of inputs to track some desired state trajectory $\{x^d_t\}_{t=0}^T$. We need an initial guess for the nominal input trajectory $\{\bar{u}_t\}^{T-1}_{t=0}$, from which the nominal state trajectory $\{\bar{x}_t\}^T_{t=0}$ can be obtained by rolling out $\{\bar{u}_t\}^{T-1}_{t=0}$ from the initial state $x_0$. For every time $t$ (i.e. in a time-varying manner), we can create a locally linear model that approximates the dynamics, with model parameters $\{\mathbf{A}_t,\mathbf{B}_t,c_t\}_{t=0}^{T-1}$ \eqref{eq:linearization}. Then, finding the optimal $\{\bar{u}_t\}^{T-1}_{t=0}$, subject to the locally linear model of the dynamics, can be written as a QP. We present the MPC variant of this problem that receives the initial state $\bar{x}_j$ at time $t=j$ and computes the optimal action for the remaining time steps,

{\small
\begin{subequations}
\label{eq:trajopt}
\begin{align}
\textbf{MPC}(\bar{x}_j) & = u^\star_j, \text{where}\\
\min_{x_t,u_t}\;\; & \norm{x_T-x_T^d}_{\mathbf{Q}_T}^2 + \sum^{T-1}_{t=j} \left(\|x_t - x^d_t\|^2_{\mathbf{Q}_t} + \|u_t\|^2_{\mathbf{R}_t}\right)\\
\text{s.t.}\;\; & x_{t+1} = \mathbf{A}_t(x_t-\bar{x}_t) + \mathbf{B}_t(u_t - \bar{u}_t) + c_t,\\
& \mathbf{C}^x_t x_t\leq d^x_t,\; \mathbf{C}^u_t u_t \leq d^u_t, \; \forall t\in\{j \cdots T-1\}, \label{eq:trajopt:x_u_constraints}\\
& x_j = \bar{x}_j.
\end{align}
\end{subequations}
}
\noindent Here, $\{\mathbf{Q}_t,\mathbf{R}_t\}$ are the quadratic weights for state and input, respectively; $\mathbf{Q}_T$ is the weight on the terminal state; $\{\mathbf{C}^x_t,d^x_t\}$ and $\{\mathbf{C}^u_t,d^u_t\}$ are inequality parameters on the state and input, respectively. The linear constraints \eqref{eq:trajopt:x_u_constraints} can enforce, for instance, joint and actuation limits.

The iMPC algorithm is summarized in Alg. \ref{alg:impc}. 
In every outer iteration (body of the $\mathtt{while}$ loop starting at Line \ref{alg:impc:while}), iMPC solves truncated versions of \eqref{eq:trajopt} for $T - 1$ times. Specifically, at inner iteration $j$ (body of the $\mathtt{for}$ loop starting at Line \ref{alg:impc:for}), we solve MPC \eqref{eq:trajopt} for the sub-problem starting at $t=j$ (Line \ref{alg:impc:mpc}), and apply $u_j^\star$ from the solution to update $\bar{x}_{j+1}$ (Line \ref{alg:impc:update_x}). During MPC, we also enforce a trust region by using \eqref{eq:trajopt:x_u_constraints} to constrain $x_t$ and $u_t$ to stay close to $\bar{x}_t$ and $\bar{u}_t$, respectively.

\begin{algorithm}
\caption{\textbf{iMPC}}\label{alg:impc}
\textbf{Input:} Initial state $x_0$, input trajectory guess $\{\bar{u}_t\}_{t=0}^{T-1}$\;
\textbf{Output:} Optimized input trajectory $\{\bar{u}_t\}_{t=0}^{T-1}$ \;
$\{\bar{x}_t\}_{t=0}^T\leftarrow$ Rollout $f$ from $x_0$ with$\{\bar{u}_t\}_{t=0}^{T-1}$\;
\While {\textrm{not converged}}{ \label{alg:impc:while}
    Compute system matrices $\{\mathbf{A}_t,\mathbf{B}_t,c_t\}_{t=0}^{T-1}$\;
    \For {$0\leq j < T$} { \label{alg:impc:for}
        $\bar{u}_j \leftarrow \mathbf{MPC}(\bar{x}_j)$ \label{alg:impc:mpc}\; 
        $\bar{x}_{j+1} \leftarrow f(\bar{x}_j,\bar{u}_j)$ \label{alg:impc:update_x}\;
    }
}
\algorithmicreturn $\; \{\bar{u}_t\}_{t=0}^{T-1}$
\end{algorithm}

To apply smoothing to iMPC, we substitute the linearizations of smooth surrogates $\{\mathbf{A}_{t,\rho},\mathbf{B}_{t,\rho},c_{t,\rho}\}^{T-1}_{t=0}$ \eqref{eq:ABc_rho} for the first-order Taylor expansions $\{\mathbf{A}_t,\mathbf{B}_t,c_t\}_{t=0}^{T-1}$ \eqref{eq:linearization}. After every outer iteration, we reduce the variance of $\rho$, allowing the smooth surrogates $f_\rho$ to converge to true dynamics $f$. For analytic smoothing, we initially set $\kappa$ using physical intuition from \eqref{eq:complementarityrelaxation}, and increase it every iteration.

\begin{figure*}
% \vspace{0.2cm}
\centering
\subfloat[\code{PlanarPushing}.]{
	\includegraphics[width=0.32\linewidth]{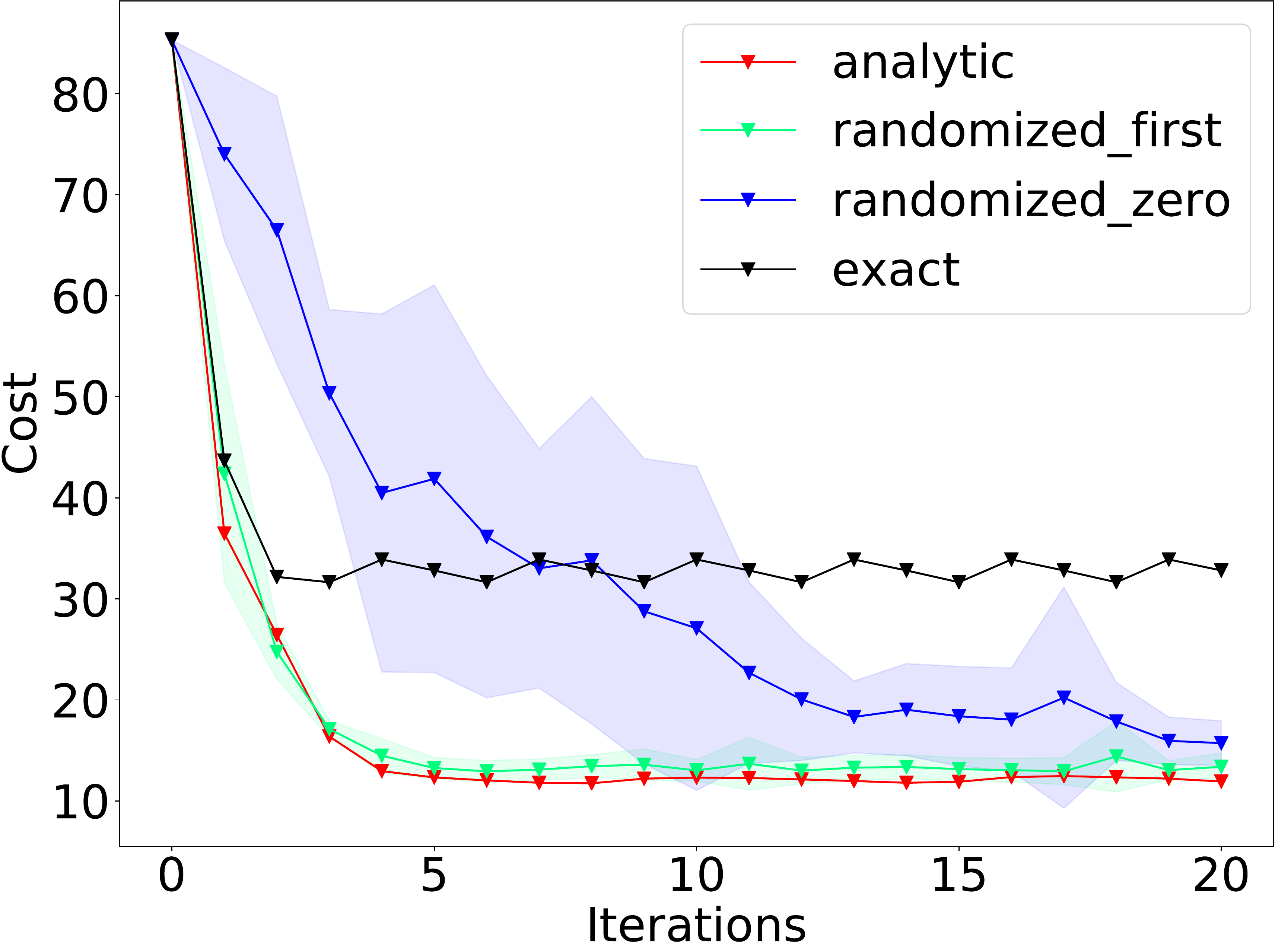}
}
\subfloat[\code{PlanarHand} Re-orientation.]{
	\includegraphics[width=0.32\linewidth]{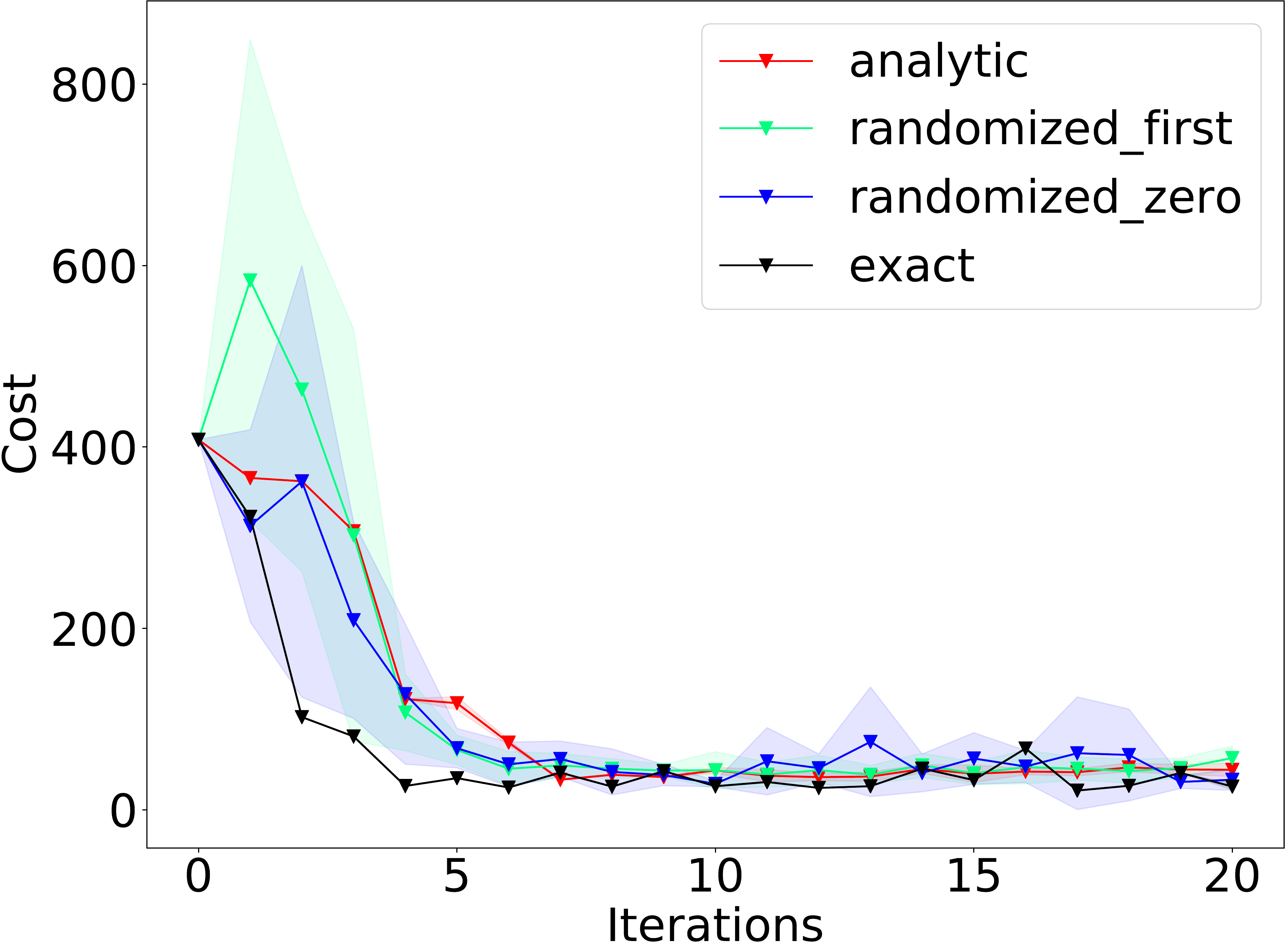}
}
\subfloat[\code{AllegroHand} Rotation.]{
	\includegraphics[width=0.32\linewidth]{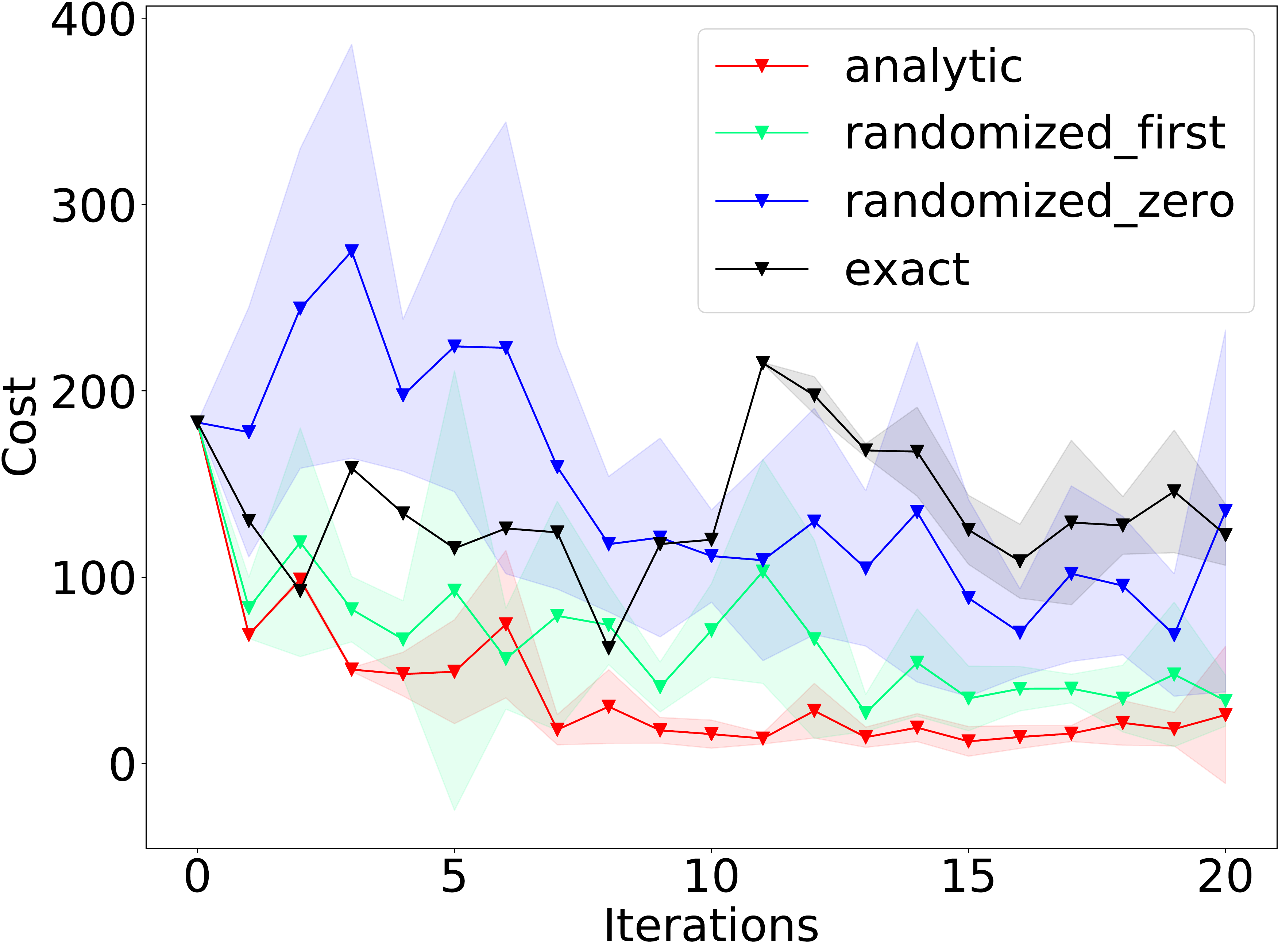}
}
\caption{Performance of iMPC with different smoothing schemes: analytic, randomized (first-order), randomized zero-order, and exact (no smoothing). For each method, the solid line represents the mean over five runs, and the shaded region represents the standard deviation. }
\label{fig:trajoptperformance}
\vspace{-0.2cm}
\end{figure*}

\subsection{Experiment Setup}\label{sec:trajopt_setup}
\subsubsection{Systems Description \label{sec:trajopt_setup:systems}}
We test iMPC with different smoothing schemes on two planar systems from \cite{bundledgradients} and a 3D system for in-hand rotation. We describe our systems below, and their visualization can be seen in Fig. \ref{fig:rrtperformance}. The three tuple after the name of each system indicates $(\nU, \nA, \nCG)$, where $\nU$ is the number of unactuated DOFs, $\nA$ the number of actuated DOFs, and $\nCG$ the number of collision geometries.

\begin{enumerate}
\item {\bf Planar Pushing, (3,2,2)}.  A classical example of nonprehensile manipulation \cite{lynch1996stable}. The goal is specified as some 2D configuration of the box.
\item {\bf Planar Hand Reorientation, (3,4,13)}.  We use a planar hand with two fingers, each with two DOFs. The goal is to change the position and orientation of the ball in a 2D plane.
% , the 2D robot pushes the object to desired position and orientation. 
\item {\bf Allegro In-Hand Rotation, (6,16,20)}.  3D In-hand rotation of the ball with the full model of the Allegro hand \cite{huang2020efficient}. The goal is specified as a rotated configuration of the ball.
\end{enumerate}

\subsubsection{Initialization}
As mentioned in Sec. \ref{sec:iMPC}, given an initial state $x_0$, we need to initialize the nominal input trajectory $\{\bar{u}_t\}^{T-1}_{t=0}$, where $\bar{u}_t$ is the commanded positions of the robots at step $t$ under the CQDC dynamics. Empirically, we find that good convergence can be achieved with a constant initialization, i.e. $\bar{u}_0 = \bar{u}_1 = \dots \bar{u}_{T-1}$. 

For the numerical experiments in this section, we need a $\bar{u}_0$ that makes contact with the object. Otherwise the baseline which does not use smoothing would have zero gradients and make no progress at all. 

In contrast, for iMPC with smoothing, it is sufficient to set $\bar{u}_0 = q^\mathrm{a}_0$, as long as $q^\mathrm{a}_0$ is not ``too far away" from making contact with the objects (the reason is explained in Example \ref{ex:planarhandreachableset}). In many practical problems, the object's initial configuration $q^\mathrm{u}_0$ is fixed, but we are free to choose the initial robot configuration $q^\mathrm{a}_0$. In this case, we can simply calculate a $q^\mathrm{a}_0$ that is ``close" to making contact using, for example, methods that compute grasps \cite{murray2017mathematical}. 

\subsubsection{Hardware \& Implementation Details}
The numerical experiments are run on a desktop with one AMD Threadripper 2950 CPU (16 cores, 32 threads) and 32GB of RAM. The code for iMPC using different smoothing schemes is identical except for the computation of the linearizations. For analytic smoothing and the baseline, we solve respectively the smoothed \eqref{eq:q_dynamics_log} and original \eqref{eq:q_dynamic_socp} dynamics once and then apply the chain rule to get the linearization. For first-order randomized smoothing, we solve the original dynamics \eqref{eq:q_dynamic_socp} and apply the chain rule for 100 samples ($N=100$), which is parallelized on all available threads, and then average the gradients of the samples. Zeroth-order randomized smoothing simply requires parallel evaluation of the dynamics.

\begin{figure}[t]
\centering
\subfloat[Initial configurations and goals. Each system is shown in its initial configuration $q_0$. The thicker frame denotes the goal while the thinner frame denotes the initial configuration of the object.]{
	\includegraphics[width=0.45\textwidth]{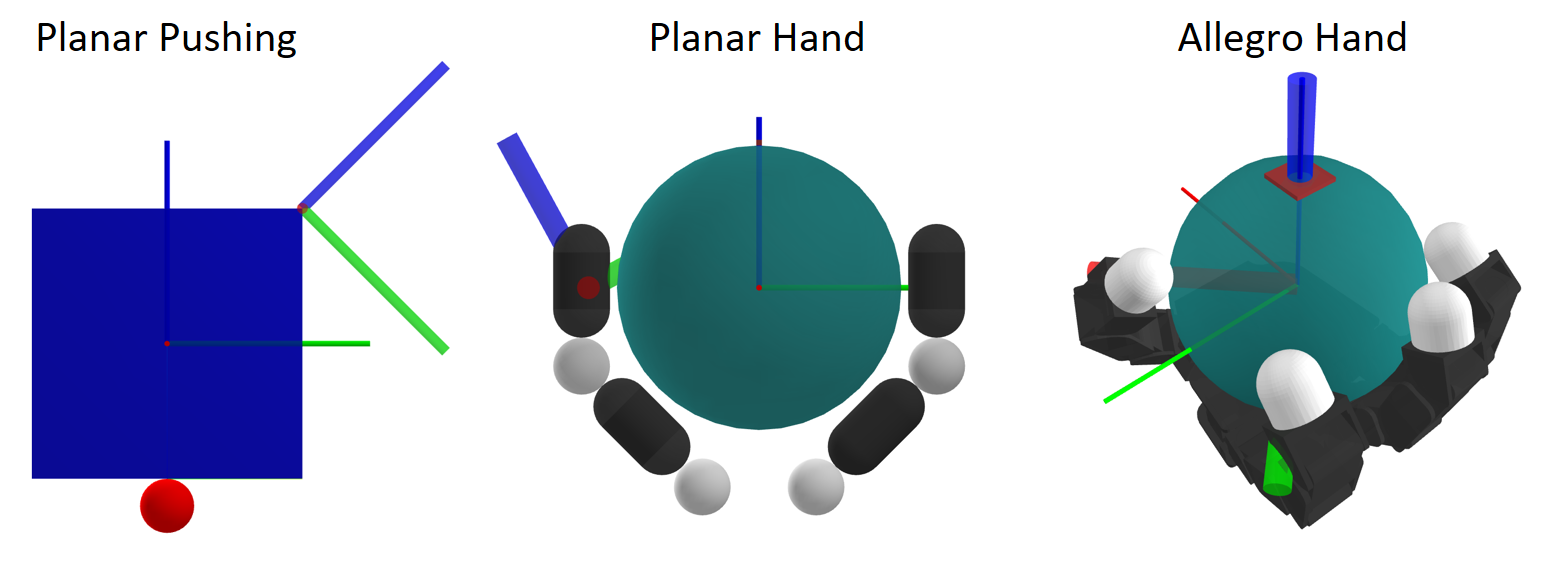}
	\label{fig:trajopttasks}
}\\
\vspace{-0.4cm}
\subfloat[Final object configuration achieved by the best runs within each of the four methods. Pink shaded denotes the goal configuration for the first two examples, while the goal configuration in the last example is marked by the pink line protruding out of the object. Colors correspond to the plots in Fig. \ref{fig:trajoptperformance}.]{
    \includegraphics[width = 0.45\textwidth]{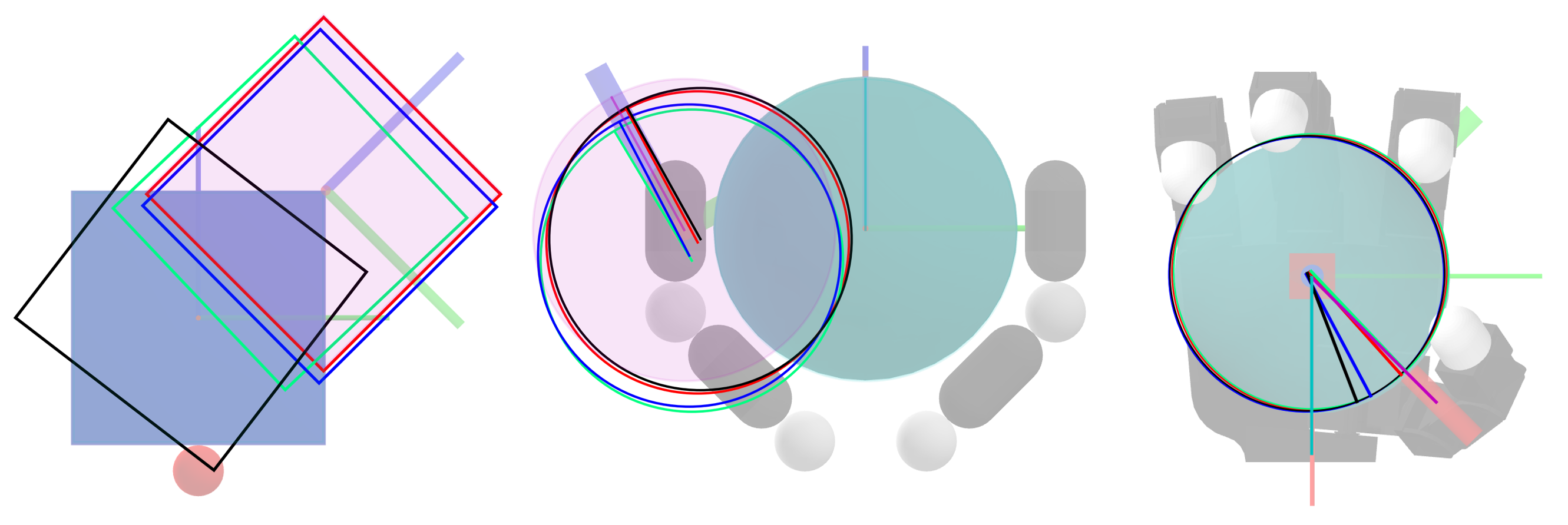}
    \label{fig:trajoptvis}
}
\caption{Tasks and results for the trajectory optimization case study.}
\label{fig:trajopt}
\end{figure}

\begin{table}[thpb] \label{tab:trajoptresults}
\centering
\begin{tabular}{|| c | r | r | r | r | r | r |} 
 \hline
 Problem & \multicolumn{2}{||c||}{PlanarPushing} & \multicolumn{2}{||c||}{PlanarHand} & \multicolumn{2}{||c||}{AllegroHand}  \\\hline
 Method & Cost & Time(s) & Cost & Time(s) & Cost & Time(s) \\ \hline\hline

    A.        & 11.74 & 2.17  & 26.55 &  5.20  & \bf{5.78} & 19.59  \\
    RF.       & \bf{11.73} & 4.64 & \bf{17.8}7 & 23.09 & 8.42 & 40.07 \\
    RZ.      & 12.86 & 4.61 & 18.29 & 11.93 & 28.21 & 34.05 \\
    Exact            & 31.64 & 1.88 & 18.49 & 5.91 & 44.68 & 12.92 \\\hline
\end{tabular}
\caption{Average of minimum cost and running time achieved by different methods. All methods are ran for $20$ iterations across 5 trials. Method abbreviations: A(Analytic), RF(Randomized First), RZ(Randomized Zero).}
\label{table:trajoptresults}
\vspace{-0.3cm}
\end{table}

%we note that we always initialize the system to be in a contacting configuration, as otherwise the gradient without smoothing is always zero in a non-contacting configuration. Consequently, algorithms that rely on exact linearization simply makes no progress. Though we initialize this way to give exact linearization an advantage, we note that the ability to converge from a non-contacting configuration is also a strength of using smooth surrogates.

\subsection{Results \& Discussion}
In Fig. \ref{fig:trajoptperformance}, we plot the performance of iMPC with a baseline that does not use smoothing, and the three different smoothing schemes in Sec. \ref{sec:smoothing_computation_schemes}, namely analytic, randomized first-order, and randomized zeroth-order. We also summarize the running time of each method, as well as the minimum cost achieved across the iterations, in Table \ref{table:trajoptresults}. Illustrations of the tasks and the results achieved by different methods are shown in Fig. \ref{fig:trajopt}. We interpret the results and discuss the relevant findings in this section. 
\subsubsection{Exact vs. Smoothing} For \code{PlanarPushing} and \code{AllegroRotation}, the various smoothing schemes achieve much lower costs than using exact gradients. However, for \code{PlanarHand}, using the exact linearization is performant as well. This difference may be explained by the observation that the planar hand example does not go through many mode changes, while the planar pusher and the Allegro hand require several mode changes to converge to the optimal trajectory.

\subsubsection{Analytic vs. Randomized Smoothing} Comparing the performance of the three smoothing schemes, the analytic and the first-order randomized smoothing perform similarly, while the zeroth-order version does not perform as well. We believe the cause lies in the high variance characteristic of the zeroth-order estimator in higher dimensions.

\subsubsection{Running (wall-clock) Time} While analytic smoothing only requires one evaluation of the smoothed dynamics \eqref{eq:q_dynamics_log} in order to compute $(\mathbf{A}_\rho,\mathbf{B}_\rho,c_\rho)$, randomized smoothing requires taking $N$ samples and averaging them, which costs $N$ times more compute-time. After parallelization, we expect randomized smoothing to be roughly $N/\xi$ times slower than analytic smoothing where $\xi$ is the number of threads. Indeed, with $N=100$ and $\xi=32$, our results show that randomized smoothing is 2 to 3 times slower than analytic smoothing.

\section{Local Mahalanobis Metric for RRT}
\label{sec:mahalanobis}
\noindent While trajectory optimization can find trajectories reaching goals that are close to the initial configuration, it is highly prone to local minima when the goal is further away (e.g. moving the box back in \code{PlanarPushing}, rotating the ball by 180 degrees in \code{PlanarHand}, \code{AllegroHand}). To solve these tasks, the planning algorithm needs to be more global. When faced with such problems, the RRT algorithm \cite{lavalle1998rapidly} has proven to be a classical and effective method for global planning. 

However, extending RRT to dynamical systems (i.e. kinodynamic RRT) has been difficult, as a distance metric between two states is hard to define. In \cite{shkolnik2009reachability}, it was argued that a good distance metric for RRT would need to explicitly consider dynamic reachability in order to efficiently grow the tree. The authors further proposed Reachability-Guided RRT (RG-RRT), which had system-specific reachability metrics that was shown to be effective for smooth systems. To alleviate the limitation of being system-specific, later works have considered building such metrics based on local characteristics of the dynamics such as local linearizations \cite{haddad2021anytime,wu2020r3t}.

However, when the dynamics involves contact, such local linearizations are no longer informative, and existing approaches often tackle dynamic reachability by explicitly considering contact modes \cite{chen2021trajectotree,cheng2021contact}. This has led to planners that scale poorly with the number of contacts. In contrast, we propose to handle the challenges brought about by contact with smoothing. We show that when combined with smoothing, the locally linear model can be used to construct an informative distance metric that is consistent with notions of reachability. 

\begin{figure*}[t]
\centering\includegraphics[width=0.98\textwidth]{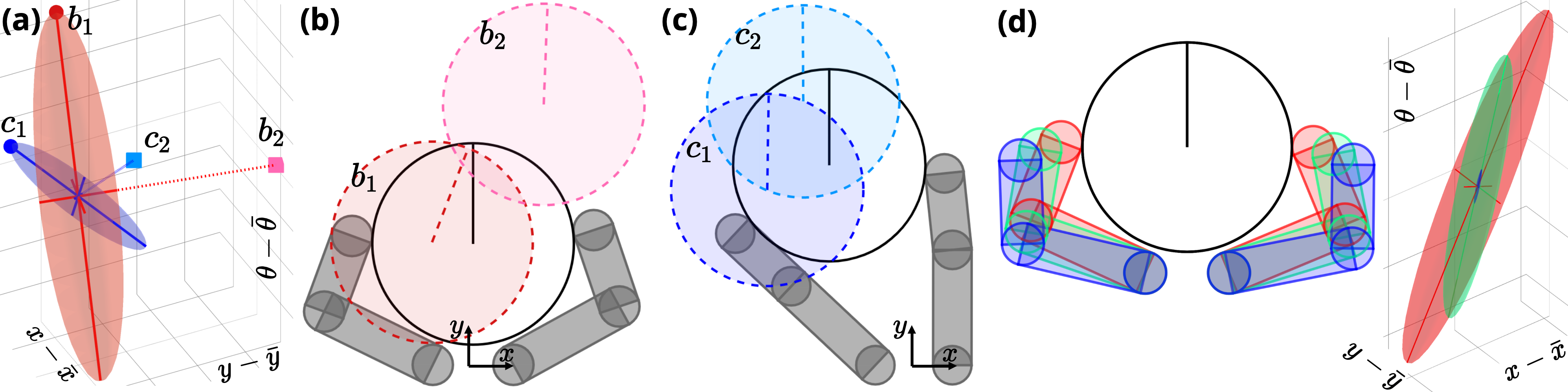}
\caption{(\textbf{a}) Two different sublevel sets $\mathcal{R}^\mathrm{u}_{\rho, \varepsilon, \gamma}$, represented as ellipsoids, shown in the space of $\qu$, with $\varepsilon = 1$, and $\gamma = 10^{-6}$. The ellipsoid centers are shifted to the origin for easy comparison. Red ellipsoid: $\mathcal{R}^\mathrm{u}_{\rho, \varepsilon, \gamma}$ for the system configuration in Fig. \ref{fig:reachability_planar_hand}b; blue ellipsoid: $\mathcal{R}^\mathrm{u}_{\rho, \varepsilon, \gamma}$ for the configuration in Fig. \ref{fig:reachability_planar_hand}c. Points $b_1$, $c_1$ are where ellipsoids' major axes intersect their boundaries. Points $b_2$, $c_2$ are points along the minor axes of the ellipsoids, and satisfy $\norm{b_1} = \norm{b_2}$ and $\norm{c_1} = \norm{c_2}$, where the norm is based on the standard Euclidean metric.
(\textbf{b}) The solid robots and objects represent the $\bar{q}$ at which the red $\mathcal{R}^\mathrm{u}_{\rho, \varepsilon, \gamma}$ in (a) is computed. The straight line on the puck indicates its orientation. The dashed dark red puck corresponds to the configuration $b_1$, and pink to $b_2$. Note that $b_1$ is easier to each than $b_2$.
(\textbf{c}) Similar to Fig. \ref{fig:reachability_planar_hand}b, the dashed dark blue puck correspond to $c_1$, and light blue to $c_2$. It is also easier to reach $c_1$ than $c_2$.
(\textbf{d}) The volume of $\mathcal{R}^\mathrm{u}_{\rho, \varepsilon, \gamma}$ shrinks as the fingers get further away from the puck. The ellipsoids on the right are color-coded to match the robot configurations on the left. Note that the blue ellipsoid is barely visible.}
\label{fig:reachability_planar_hand}
\vspace{-0.5cm}
\end{figure*}

\subsection{The Local Mahalanobis Metric}

Consider the following problem: given the current configuration $\bar{q}$, and some queried configuration $q$, how can we formulate a distance metric $d(q;\bar{q})$ that is consistent with reachability characteristics of the system? We propose to utilize the locally linear model around the nominal configuration $\bar{q}$, that characterizes the local response of the next system configuration $q_+$ with respect to the movement of the actuated configurations $u$. This local model can be written as 
\begin{equation}
    q_+ = \mathbf{B}(\bar{q},\bar{q}^{\mathrm{a}})\underbrace{(u - \bar{q}^\mathrm{a})}_{\delta u} + c(\bar{q},\bar{q}^{\mathrm{a}})
\end{equation}
where the notation is consistent with the CQDC dynamics formulation in Sec.\ref{sec:convex_quasi_dynamic_contact_dynamics}; the input $u$ is the position command to the system, and $\mathbf{B},c$ are defined as \eqref{eq:linearization}. We note that the contribution of $\mathbf{A}$ term is zero since $\delta x = 0$. Also note that we set the nominal action to be  $\bar{u}=\bar{q}^{\mathrm{a}}$.

Given such a local characterization, and a queried state $q$, we define the Mahalanobis metric as follows.
\begin{definition}
{\bf Local Mahalanobis Metric.}\normalfont \label{def:mahalanobis}
Given a nominal configuration $\bar{q}$, and queried configuration $q$, we define the Mahalanobis distance $d_\gamma$ of $q$ from $\bar{q}$ as follows:
\begin{equation}
\label{eq:metric}
\begin{aligned}
d_{\gamma}(q;\bar{q}) & \coloneqq \|q - \mu\|_{\mathbf{\Sigma}^{-1}_{\gamma}}=\textstyle\frac{1}{2}(q-\mu)^\intercal\mathbf{\Sigma}^{-1}_\gamma (q-\mu) \\
\mathbf{\Sigma}_{\gamma} & \coloneqq \mathbf{B}(\bar{q},\bar{q}^\mathrm{a})\mathbf{B}(\bar{q},\bar{q}^\mathrm{a})^\intercal + \gamma\mathbf{I}_n, \; \mu \coloneqq c(\bar{q},\bar{q}^\mathrm{a}).
\end{aligned}
\end{equation}
\end{definition}

The regularization $\gamma\mathbf{I}_n$ is added to ensure that $\mathbf{\Sigma}_{\gamma}$ is positive definite and the inverse $\mathbf{\Sigma}^{-1}_{\gamma}$ is well-defined. Note that the $\varepsilon$-sublevel set of this metric $d_\gamma(q;\bar{q})$, which we denote by $\mathcal{R}_{\varepsilon,\gamma}(\bar{q})$, describes an \emph{ellipsoid} that is centered at $\mu$ and has a shape matrix $\mathbf{\Sigma}^{-1}_\gamma$. We further motivate our construction of the metric by noting that this ellipsoid can be alternatively characterized (equivalent up to regularization) by the set
\begin{equation}
    \mathcal{R}_{\varepsilon}(\bar{q})\coloneqq \left\{\mathbf{B}(\bar{q},\bar{q}^{\mathrm{a}}) \delta u + c(\bar{q},\bar{q}^\mathrm{a})\; | \; \|\delta u\|\leq \varepsilon\right\}.
\end{equation}

This equivalence relation is exact when $\mathbf{B}$ has full row rank (i.e. the system is one-step controllable) and $\gamma=0$. On the other hand, if $\mathbf{B}$ loses rank, one of the principle axis of $\mathcal{R}_{\varepsilon}$ has a length of zero and the set becomes degenerate.

Finally, note that when $q-\mu \notin \mathbf{Range}(\mathbf{B})$, i.e. there is no actuation $u$ that can take the state $\bar{q}$ to the queried state $q$, the distance $ d_{\gamma}(q;\bar{q})$ is a large number dominated by the inverse of the regularization term $\gamma^{-1}$, which is consistent with the intuition that states that are harder to reach are further away.

\subsection{Metric on Smoothed Dynamics and Unactuated Objects}

As explained in the previous sections, the local model constructed using $\mathbf{B}$ may not be a very informative one for non-smooth systems with contact. In light of the various smoothing schemes introduced in Sec.\ref{sec:localsmoothing} to alleviate this issue, we propose a metric by utilizing the linearization of the smooth surrogate $(\mathbf{B}_\rho,c_\rho)$ as described in \eqref{eq:ABc_rho}, as opposed to those of the original contact dynamics, $(\mathbf{B},c)$. 

Furthermore, for systems where robots interact with unactuated objects through contact, we focus on the reachability of the objects, as the robots are actuated and can easily move to a desired configuration without contact. We combine smoothing and the object-centric reachability in the following variant of the Mahalanobis metric $d^\mathrm{u}_{\rho,\gamma}$, 
\begin{equation}
\begin{aligned}\label{eq:unactuated_mahalanobis_metric}
d_{\rho,\gamma}^\mathrm{u}(q;\bar{q}) & \coloneqq \|q^\mathrm{u} - \mu^\mathrm{u}_\rho\|_{\mathbf{\Sigma}^{\mathrm{u}^{-1}}_{\rho,\gamma}}, \\
\mathbf{\Sigma}_{\rho,\gamma}^\mathrm{u} & \coloneqq \mathbf{B}^\mathrm{u}_\rho(\bar{q},\bar{q}^\mathrm{a})\mathbf{B}^\mathrm{u}_\rho(\bar{q},\bar{q}^\mathrm{a})^\intercal + \gamma\mathbf{I}_{n_\mathrm{u}},\\ 
\mu_\rho^\mathrm{u} & \coloneqq c_\rho^\mathrm{u}(\bar{q},\bar{q}^\mathrm{a}).
\end{aligned}
\end{equation}
where $\mathbf{B}^\mathrm{u}_\rho$ is formed by the rows of $\mathbf{B}_\rho$ corresponding to the unactuated DOFs, and $c_\rho^\mathrm{u}$ is defined similarly. Finally, we define $\mathcal{R}_{\rho,\varepsilon,\gamma}^\mathrm{u}(\bar{q})$ as the $\varepsilon$-sublevel set of $d^\mathrm{u}_{\rho,\gamma}(q;\bar{q})$. 

In the rest of this section, we give several examples that provide intuition into the local Mahalanobis metric $d^\mathrm{u}_{\rho,\gamma}$ and its sublevel set $\mathcal{R}^\mathrm{u}_{\rho,\varepsilon,\gamma}$. 

\begin{example} \normalfont \textbf{(Understanding $\B$ for Planar Systems)}
When both the robots and objects are constrained in a plane (e.g. the planar systems in Sec. \ref{sec:trajopt_setup:systems}), the CQDC dynamics \eqref{eq:q_dynamic_socp} simplifies to the following QP:
\begin{subequations}
\label{eq:q_dynamics_planar_qp}
\begin{align}
\underset{\dq}{\minimize} \; &\frac{1}{2} \dq^\intercal \mathbf{Q} \dq + b^\intercal \dq, \; \text{subject to} \\
&(\Jn[i] + \mu_i \Jt[i]) \dq + \phi_i \geq 0, \; i \in \{1\dots\nC\}, \label{eq:q_dynamics_planar_qp:constraint1}\\
&(\Jn[i] - \mu_i \Jt[i]) \dq + \phi_i \geq 0, \; i \in \{1\dots\nC\}, \label{eq:q_dynamics_planar_qp:constraint2}
\end{align}
\end{subequations}
where the contact Jacobian $\Jt[i]$ has only one row instead of two, and the conic contact constraint \eqref{eq:q_dynamic_socp:constraint} reduces to two inequality constraints \eqref{eq:q_dynamics_planar_qp:constraint1} and \eqref{eq:q_dynamics_planar_qp:constraint2}. We define $\J \in \R[(2\nC) \times n_q]$ by stacking the $\Jn[i] + \mu_i \Jt[i]$ and $\Jn[i] - \mu_i \Jt[i]$ from \eqref{eq:q_dynamics_planar_qp:constraint1} and \eqref{eq:q_dynamics_planar_qp:constraint2} into a single matrix, and partition $\J$ into $\Ju$ and $\Ja$ in a similar way as in \eqref{eq:contact_jacobian_i}.

More structure behind the $\mathbf{B}$ matrix (as defined in \eqref{eq:q_dynamics_AB}) can be revealed with a bit of linear algebra. We can work out by hand the application of the implicit function theorem to the KKT conditions of \eqref{eq:q_dynamics_planar_qp}, and the chain rule in \eqref{eq:DdqDu}, to obtain an explicit expression for $\B$:
\begin{subequations}
\label{eq:planar_B_structure}
\begin{align}
\mathbf{B} &=
\begin{bmatrix}
\Ba \\ 
\Bu
\end{bmatrix}
=
\begin{bmatrix}
\mathbf{I} - (h^2\Ka)^{-1} (\JaActive)^\intercal \mathbf{P} \JaActive \\
\Mu^{-1} (\JuActive)^\intercal \mathbf{P} \JaActive
\end{bmatrix}, \; \text{with}\\
\mathbf{P} &= \left[\JuActive \Mu^{-1} (\JuActive)^\intercal + \JaActive (h^2 \Ka)^{-1} (\JaActive)^{\intercal} \right]^{-1}.
\end{align}
\end{subequations}
where we assume $\JuActive$ and $\JaActive$ have full row rank. The tilde over a Jacobian indicates the sub-matrix formed by rows of the original matrix corresponding to the active constraints, i.e. contacts with non-zero contact forces.

%The robot actuation matrix, $\Ba$, is the difference between identity and a term that decreases as $\Ka$ grows, which is intuitively saying that a stiff robot usually goes to where it is asked to go. As for $\Bu$, it is noteworthy that the possible object motions in a specific contact mode live in the range of $(\JuActive)^\intercal$, the transpose of the active contact Jacobian. 

The structure in $\Bu$ explains why $\Bu_\rho$ is a good measure of the object's reachability when there is contact.
We can interpret $\mathbf{Range}(\JuActive^\intercal)$ as achievable object motions under the specific subset of active contacts.
By averaging $\Bu$ computed from different contacts which can be activated from the nominal $(\bar{q}, \bar{u})$, $\Bu_\rho$ summarizes possible object motions due to contact, in the form of $\mathbf{Range}(\JuActive^\intercal)$ weighted by $\Mu$ and $\mathbf{P}$. 

Furthermore, for a configuration with no active contacts, \eqref{eq:planar_B_structure} implies that $\Ba = \I$ and $\Bu = \mathbf{0}$, as both $\JaActive$ and $\JuActive$ are empty matrices in the absence of active contacts. This has the intuitive interpretation that under a $u$ that does not lead to contacts, the robot will move to where it is commanded to, and the object will remain still.

As $\mathbf{B}_\rho^\mathrm{u}$ is the expected value of $\mathbf{B}^\mathrm{u}$, it follows naturally from the above observation that the local distance metric $d_{\rho, \gamma}^\mathrm{u}$ tends to be dominated by the regularization $\gamma \mathbf{I}_{n_\mathrm{u}}$ for a nominal configuration $\bar{q}$ where robots and objects are far from making contact. In such cases, the probability that an action $u$ sampled from a distribution $\rho$ centered at $\bar{q}^\mathrm{a}$ leads to active contacts is low. As a result, in the Monte-Carlo estimation of $\mathbf{B}_\rho^\mathrm{u}$, such samples simply introduce $\mathbf{0}$ into the average, dragging the distance metric $d^\mathrm{u}_{\rho,\gamma}$ towards being dominated by the regularization.
\end{example}

\begin{example} \normalfont \textbf{(Metric on Planar Hand)}
\label{ex:planarhandreachableset}
We illustrate how the Mahalanobis metric can guide planning using the \code{PlanarHand} system first introduced in Sec. \ref{sec:trajopt_setup}.
As shown in Fig. \ref{fig:reachability_planar_hand}, the system lives in the $xy$ plane, with gravity pointing into the paper along the negative $z$ direction. The system consists of two actuated 2-link robotic fingers and an unactuated puck which is free to translate and rotate. Each finger can interact with the ball through frictional contacts along both links.

For a given $\qu$, the difficulty of reaching $\qu$ from $(\bar{q}, \bar{u})$ can be measured by the local Mahalanobis metric $d^\mathrm{u}_{\rho, \gamma}$, whose $1$-sublevel sets are shown in Fig. \ref{fig:reachability_planar_hand}a as ellipsoids. Although object configurations $b_1$ and $b_2$ are equidistant to the origin under the globally-uniform Euclidean metric, $b_1$ is considered much closer than $b_2$ under the local Mahalanobis metric (red ellipsoid). Indeed, in Fig. \ref{fig:reachability_planar_hand}b, reaching $b_1$ from the current puck configuration seems easier than reaching $b_2$. A similar observation can be made for the configuration in Fig. \ref{fig:reachability_planar_hand}c.

In addition, the local Mahalanobis metric also varies greatly from one configuration another, as evidenced by the difference between the blue and red ellipsoids in Fig. \ref{fig:reachability_planar_hand}a. This implies that a globally-uniform metric is rarely a good measure of reachability characteristics.

Lastly, the ellipsoid that corresponds to the $1$-sublevel set shrinks as the nominal state gets further away from the contact manifold, as shown in Fig. \ref{fig:reachability_planar_hand}d. This signifies that the configurations where the object is less accessible by the robot are naturally considered ``further away" and can thus be avoided by the planner.
\end{example}

\section{RRT through Contact \label{sec:rrt_for_contact}}
\noindent We are now ready to present our smoothing-based enhancements to the vanilla RRT algorithm, which we reproduce in Alg. \ref{alg:rrt} to establish notations for our discussion. Our method enhances RRT by incorporating
(\textbf{i}) a reachability-aware $\Nearest$ operation based on the smoothed Mahalanobis metric on the unactuated objects $d_{\rho,\gamma}^\mathrm{u}$,
% a $\mathtt{Nearest}$ operation based on the object Mahalanobis reachability metric $d_{\rho,\gamma}^\mathrm{u}$ of individual nodes constructed on the smoothed dynamnics;
(\textbf{ii}) a fast $\Extend$ operation based on the projection of the subgoal to the range of $\mathbf{B}_\rho$; and
(\textbf{iii}) a contact sampling procedure which improves the reachability of nodes added to the tree.

We denote the RRT tree as $\mathcal{T}=(\mathcal{V},\mathcal{E})$ with vertex set $\mathcal{V}$ and edge set $\mathcal{E}$. Each node $q \in \mathcal{V}$ is simply a point in the configuration space of the system. 

\begin{algorithm}[h]
\caption{\textbf{RRT}}\label{alg:rrt}
\textbf{Input:} $q_{\mathrm{init}}, q_{\mathrm{goal}}, K$\;
\textbf{Output:} $\mathcal{T}$\;
$\mathcal{T} = \{q_{\mathrm{init}}\}$\;
\For {$k = 1, \dots, K$}{
    % \While{$\mathtt{IsInvalid}(v_{\mathrm{nearest}})$}
    $q_{\mathrm{subgoal}} = \mathtt{SampleSubgoal(p)}$\;
    $q_{\mathrm{nearest}} = \mathtt{Nearest}(q_{\mathrm{subgoal}})$\; \label{alg:rrt:nearest}
    % \EndWhile
    $q_{\mathrm{new}} = \mathtt{Extend}(q_{\mathrm{nearest}}, q_{\mathrm{subgoal}})$\;  \label{alg:rrt:extend}
    % \State $\mathcal{T}.\mathtt{Rewire(v_{\mathrm{new}})}$
    $\mathtt{AddNode}(q_{\mathrm{new}})$ \label{alg:rrt:add_node}\;
    \If{\texttt{GoalReached}}{$\textbf{break}$\;}
}
\end{algorithm}
\vspace{-0.8cm}

\subsection{Nearest Node using Local Mahalanobis Metric}
As illustrated in Sec.\ref{sec:mahalanobis}, in particular by Example \ref{ex:planarhandreachableset}, a globally-uniform metric used by the vanilla RRT is usually a poor measure of reachability. Given a subgoal $q_{\mathrm{subgoal}}$, if the nearest node $q_{\mathrm{nearest}}$ is chosen under a globally-uniform metric, reaching $q_{\mathrm{subgoal}}$ from $q_{\mathrm{nearest}}$ may require large $u$ or even be dynamically infeasible. This will compromise RRT's ability to explore the configuration space, as trying to $\mathtt{Extend}$ towards a hard-to-reach $q_{\mathrm{subgoal}}$ typically returns a child node that is close to the parent node $q_{\mathrm{nearest}}$. In order to retain RRT's ability to efficiently explore under dynamics constraints, we use the smoothed Mahalanobis metric \eqref{eq:unactuated_mahalanobis_metric} instead of the usual Euclidean metric in the $\mathtt{Nearest}$ step:
\begin{equation}
q_{\mathrm{nearest}} = \text{argmin}_{q \in \mathcal{V}}\; d^{\mathrm{u}}_{\rho,\gamma}(q_{\mathrm{subgoal}}; q).
\end{equation}

\subsection{Dynamically Consistent Extension}
After choosing $q_{\mathrm{nearest}}$ from the tree $\mathcal{T}$, we need an action or a sequence of actions that moves the system from $q_{\mathrm{nearest}}$ to $q_{\mathrm{subgoal}}$ subject to the dynamics constraint. One feasible strategy to connect $q_{\mathrm{nearest}}$ to $q_{\mathrm{subgoal}}$ is to solve for an input sequence $\{u_t\}_{t=0}^{T-1}$ using a trajectory optimization algorithm such as Alg. \ref{alg:impc} \cite{karaman2010optimal}. However, the high computational cost of trajectory optimization motivates us to seek a simpler solution.

Fortunately, as a result of the farsightedness of quasi-dynamic models, even an input sequence with $T = 1$ (i.e. a single time step) can steer the system fairly far away from $q_{\mathrm{nearest}}$. Although trajectory optimization with $T > 1$ can explore a larger region around $q_{\mathrm{nearest}}$, we find in practice that a single time step is sufficient for $\Extend$ to effectively grow the RRT tree $\mathcal{T}$.

We present the modified $\Extend$ that uses a single time step in Alg. \ref{alg:extend}. The input $u$ is computed by projecting $(q^\mathrm{u}_{\mathrm{subgoal}} - \mu_\rho^\mathrm{u})$ to $\mathbf{Range}(\Bu_\rho)$ \eqref{eq:ABc_rho}, \eqref{eq:unactuated_mahalanobis_metric} using least-squares (Line \ref{alg:extend:lstsq}). Afterwards, we normalize the input and multiply it by some stepsize $\varepsilon$. The scaled input is then passed to the forward dynamics to obtain a new node. Crucially, we use the actual dynamics $f$ as opposed to the smooth surrogate dynamics $f_\rho$ (Line \ref{alg:extend:rollout}). This ensures that while the search for the next action relies on the smoothed model, the actual path is dynamically consistent under the original non-smooth contact dynamics (i.e. CQDC). 

\begin{algorithm}
\caption{$\mathtt{Extend}$}\label{alg:extend}
\textbf{Input:} $q_{\mathrm{nearest}}, q_{\mathrm{subgoal}}$\;  \textbf{Output:} $q_{\mathrm{new}}$\;
$\delta u^\star = \mathrm{argmin}_{\delta u} \|\Bu_\rho \delta u + c_\rho^\mathrm{u} - q_{\mathrm{subgoal}}^\mathrm{u} \|$ \label{alg:extend:lstsq} \;
% $\varepsilon = \min(\varepsilon_\mathrm{max}, \norm{\delta u^\star})$ \label{alg:extend:cap}\;
\algorithmicreturn  $\; f(q_{\mathrm{nearest}}, q_{\mathrm{nearest}}^\mathrm{a} + \varepsilon \cdot \delta u^\star / \|\delta u^\star\|)$ \label{alg:extend:rollout}\;
\end{algorithm}
\vspace{-0.8cm}

\subsection{Contact Sampling}\label{sec:contactsampling}
A node $q$ where robots and objects are far from making contacts hinders the growth of the RRT tree for two reasons. First, such nodes are considered far away from most sampled subgoals under the local Mahalanobis distance metric (Fig. \ref{fig:reachability_planar_hand}d). As a result, adding such nodes to the tree hardly improves coverage of the state space. Moreover, in the event that such a node is chosen by $\Nearest$, the $\Extend$ operation that follows often results in other unfavorable configurations.

To reduce the number of such nodes in $\mathcal{T}$ and encourage exploration during tree growth, the $\Extend$ operation is replaced, with some probability, by a new operation called $\ContactSample$. $\ContactSample$ takes $q_\mathrm{nearest}$ as input, and creates another node with a better local metric by fixing $q_\mathrm{nearest}^\mathrm{u}$ and finding an informative $q_\mathrm{nearest}^\mathrm{a}$ that makes contact. 

The $\ContactSample$ operation is essential for adequate exploration of the robot's state space, and needs to be designed differently for different robots. Here we summarize our implementation of the contact samplers for the systems in Fig. \ref{fig:rrt_tasks}, which are used in our experiments in later sections. In \code{PlanarPusher} and \code{PlanarHand}, we can sample contact points on the object surface and solve for robot joint angles using inverse kinematics. For systems with the Allegro hand, we pick a random direction in the hand configuration space using EigenGrasps \cite{eigengrasp}, and simulate closing the hand along that direction until contact is established.

\begin{figure*}[t]
\centering\includegraphics[width = 0.9\textwidth]{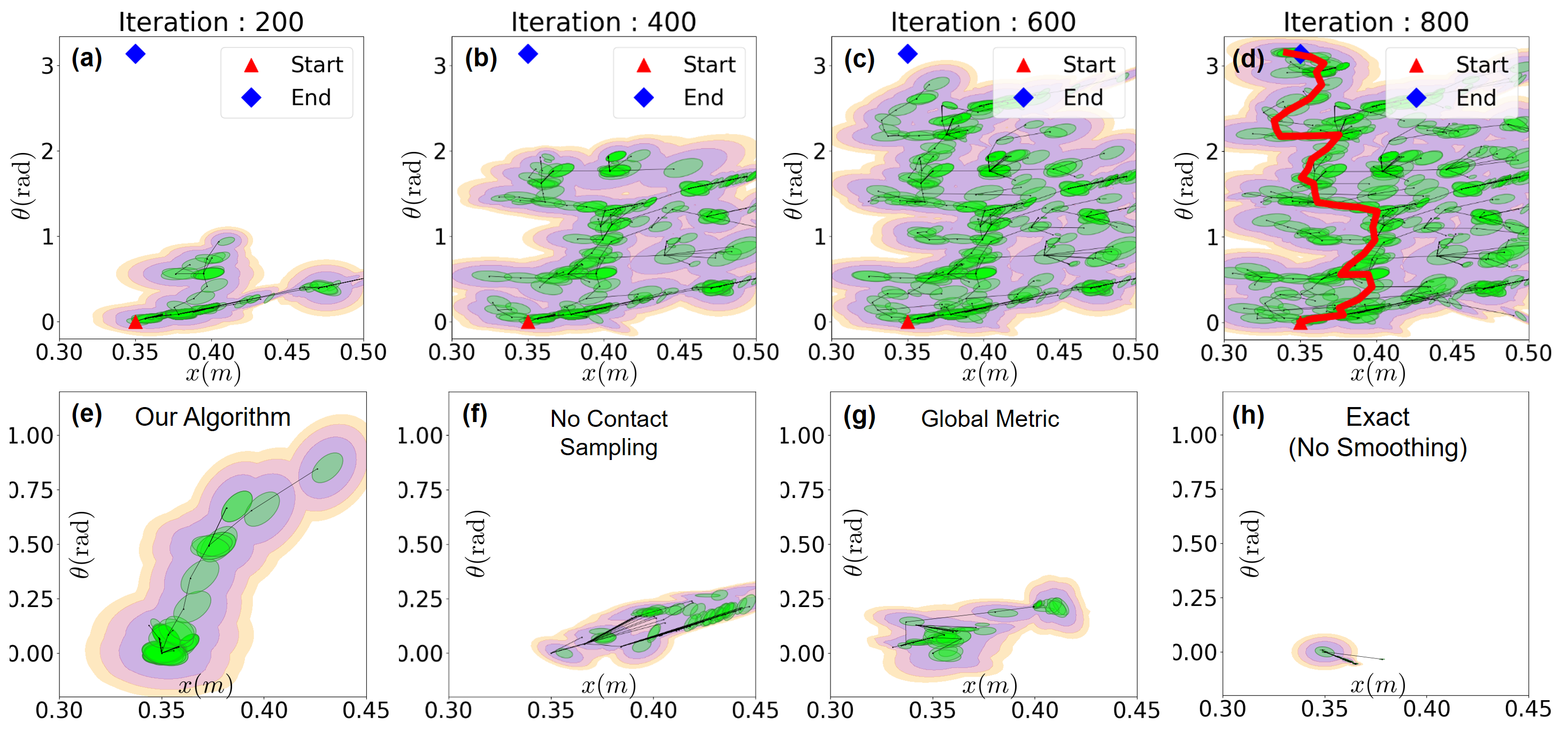}
\caption{\textbf{(a-d)} RRT trees, shown in the space of $\qu$, at different iterations of a complete run of the enhanced RRT for the \code{PlanarHandFixedY} system. The contours are the sub-level sets of the local Mahalanobis metric of the nodes. The path from start ($q_\mathrm{init}$) to goal ($q_\mathrm{goal}$) is highlighted in red in the final tree of \textbf{(d)}. \textbf{(e-h)} Visualization of RRT trees with the same number of nodes (50) but grown with different methods. (\textbf{e}) Tree grown with our algorithm; (\textbf{f}) without contact sampling; (\textbf{g}) using a globally uniform weighted Euclidean metric; (\textbf{h}) using exact gradients without smoothing. Note that our method achieves the best coverage of the space of $\qu$.}
\label{fig:degenerate}
\vskip -0.2 true in
\end{figure*}

Contact sampling introduces non-physical behavior where the robot teleports from one configuration to another. This is not a problem when the object can sustain static equilibrium without the actuated DOFs that need to teleport. For instance, in \code{AllegroHand}, when the ball is supported by the palm, the fingers are free to move around to regrasp the ball. In contrast, if \code{AllegroHand} were facing downwards, the ball would fall under gravity if it were not secured by some of the fingers. Although this is a limitation of our current implementation of contact sampling, we believe this can be resolved by a more sophisticated contact sampler which moves some of the actuated DOFs while keeping the object in static equilibrium with the rest.

\subsection{Effectiveness of Proposed Enhancements \label{sec:rrt_contact_effectiveness}}
We introduce a new system with a 2-dimensional object configuration space to illustrate the effectiveness of the proposed RRT enhancements:
\begin{itemize}
\item {\bf Planar Hand with fixed $y$, (2,4,13)}. A simplified version of the \code{PlanarHand} system in Sec. \ref{sec:trajopt_setup}. We fix the $y$-coordinate of the object, so that $\qu = (x, \theta)\in \R[2]$ can be easily plotted on paper. 
\end{itemize}

As shown in Fig. \ref{fig:degenerate}e, the vanilla RRT enhanced with the proposed $\Nearest$, $\Extend$ and $\ContactSample$ achieves good coverage of the space of $\qu$, which is crucial for RRT to adequately explore the configuration space and find a path to $q_\mathrm{goal}$. In contrast, tree growth is significantly hindered without contact sampling (Fig. \ref{fig:degenerate}f) or the local metric (Fig. \ref{fig:degenerate}g).

We also illustrate how the tree grows throughout a complete run of the enhanced RRT in Fig. \ref{fig:degenerate}a-d. Even with the proposed enhancements, tree growth can get stuck at times. This is characterized by a specific type of subgraph of the tree which we call a ``broom". A broom consists of one parent node with many child nodes, and is formed by repeated unsuccessful attempts to grow towards different subgoals from the same parent node. The occasional appearance of brooms is a sign that the proposed enhancements are not perfect. Nevertheless, the enhanced RRT is able to quickly branch out into empty part of the configuration space, and sufficiently cover the space as the tree grows. 

\subsection{Final Path Refinement}\label{sec:rrt_refinement}
The final path returned by the RRT algorithm is visually plausible, yet suffers from two minor drawbacks: (\textbf{i}) RRT tends to produce randomized paths that can be shortened, and (\textbf{ii}) the big step size used in the $\Extend$ operation creates some non-physical artifacts due to Anitescu’s convex relaxation of the Coulomb friction model (Sec.\ref{sec:convex_quasi_dynamic_contact_dynamics}). 

To mitigate these issues, we refine the RRT plan using trajectory optimization \cite{lgp,terry} and short-cutting \cite{shortcutting}. 
We first divide the RRT path into segments punctuated by $\ContactSample$ operations. We call these segments contact-rich as they involve contact-based interactions between the object and the robot. 
We shortcut the sequence of trajectories by (\textbf{i}) removing consecutive $\ContactSample$ steps, and (\textbf{ii}) truncating each segment if there is no movement in $\qu$. 
Then, for each contact-rich segment, we run trajectory optimization (Alg.\ref{alg:impc}) with a smaller time step $h$, using the RRT path segment as the initial guess. This not only smooths the final path, but also ensures that each trajectory segment is more physically realistic.
Finally, we connect adjacent contact-rich segments with a collision-free robot trajectory created by a collision-free RRT. We assume that the object configuration remains unchanged during the collision-free segment.

We find that combining these two strategies is effective in creating shorter and more physically realistic trajectories. 

\begin{figure*}[thpb]
\centering\includegraphics[width = 0.85\linewidth]{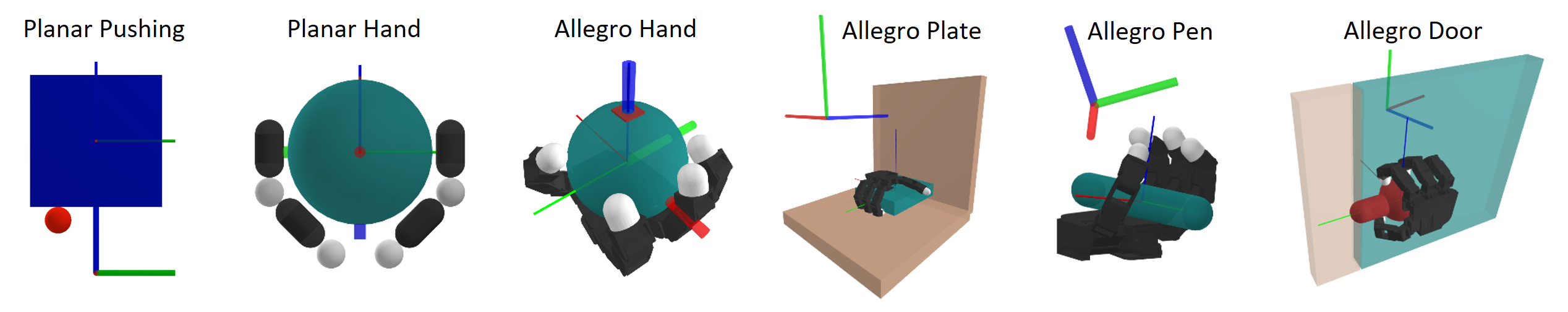}
\caption{Tasks for RRT. Similar to Fig. \ref{fig:trajoptvis}, the thicker frame denotes the goal, and the thinner frame the initial configuration of the object.}
\label{fig:rrt_tasks}
\vskip -0.2 true in
\end{figure*}

\section{Planning Results \& Discussion \label{sec:rrt_results}}
\noindent In this section, we apply our algorithm on difficult 3D contact-rich manipulation problems previously only tackled by heavy offline approaches in RL \cite{rajeswaran2018learning, chen2022system}, and illustrate that we can generate plans on the order of a minute of online compute time, all on the CPU, which shows the efficacy of our method.
The experiments in this section are designed to validate the following four hypotheses.
\begin{enumerate}
\item Using smooth surrogate greatly improves the performance over using the exact dynamics.
\item The equivalence of smoothing schemes establishes that analytic and randomized smoothing will have similar levels of performance empirically, with analytic smoothing showing superior computation time.
\item Using the Mahalanobis distance metric improves performance over a globally uniform distance metric.
\item Contact sampling greatly aids sample efficiency.
\end{enumerate}

\subsection{Experiment Setup \label{sec:rrt_experiment_setup}} 
To test the efficacy of our algorithm and our hypotheses, we run our algorithm to reach more challenging goals than the trajectory optimization examples in Sec. \ref{sec:trajopt_setup:systems}, as well as on 3 more contact-rich tasks on 3 new systems defined below.
\begin{enumerate}
    \item {\bf Pen Placement (6,19,24)}. The robot hand move the pen \cite{rajeswaran2018learning} to the desired configuration.
    \item {\bf Plate Pickup (6,19,42)}. The robot has to exploit the external contact between the plate and the wall \cite{cheng2021contact}, showing extrinsic dexterity \cite{extrinsic}.
    \item {\bf Door Opening (2,19,22)} \cite{rajeswaran2018learning} involves reasoning about a constrained system, where the handle must be rotated first before the door can be pushed open.
\end{enumerate}
The definition of the number tuples is identical to Sec.\ref{sec:trajopt_setup:systems}.

The contact-rich planning tasks are illustrated in Fig. \ref{fig:rrt_tasks}.
We design the tasks so that solving any of them with a single run of trajectory optimization is expected to fail due to the non-convexity of the problem.

To compare our algorithm with different baselines, we rate the quality of planners using two metrics:
\begin{enumerate}
    \item {\bf Iteration vs. Minimum distance to goal}.  we measure the distance between the goal and the tree, defined by $\min_{q\in\mathcal{V}} \|q^\mathrm{u} - q_\textrm{goal}^\mathrm{u}\|$ for every iteration. A successful planning algorithm would eventually reach the vicinity of the goal asymptotically, driving this metric to zero. 
    \item {\bf Iteration vs. Packing Ratio}. To characterize the \emph{exploration} performance of RRT, we do a Monte-Carlo estimation of the \emph{packing ratio}, which is defined as the volume of the space occupied by the reachability ellipses, divided by the total volume of some workspace limit for the unactuated objects. The workspace limit is the set from which subgoals are sampled when running RRT. More formally, we define the numerator as
    \begin{equation}
        V_{\mathrm{reachable}} = \textbf{vol}\big(\{q^\mathrm{u} | \min_{\bar{q}\in\mathcal{V}} d^\mathrm{u}_{\rho,\gamma}(q; \bar{q}) \leq \eta\}\big)
    \end{equation}
    where $\eta$ is some threshold on the distance metric. The Monte-Carlo estimate of $V_{\mathrm{reachable}}/V_{\mathrm{workspace}}$ can be computed by drawing $N$ samples within the workspace and counting how many of them belong to $V_{\mathrm{reachable}}$. A good planner should asymptotically reach a ratio of $1$ if the system can reach all points in the object's workspace from which $q_{\mathrm{subgoal}}^\mathrm{u}$ are sampled.
\end{enumerate}

\begin{figure*}[htp]
\centering
\subfloat[Planar Pushing.]{
	\includegraphics[width=0.485\linewidth]{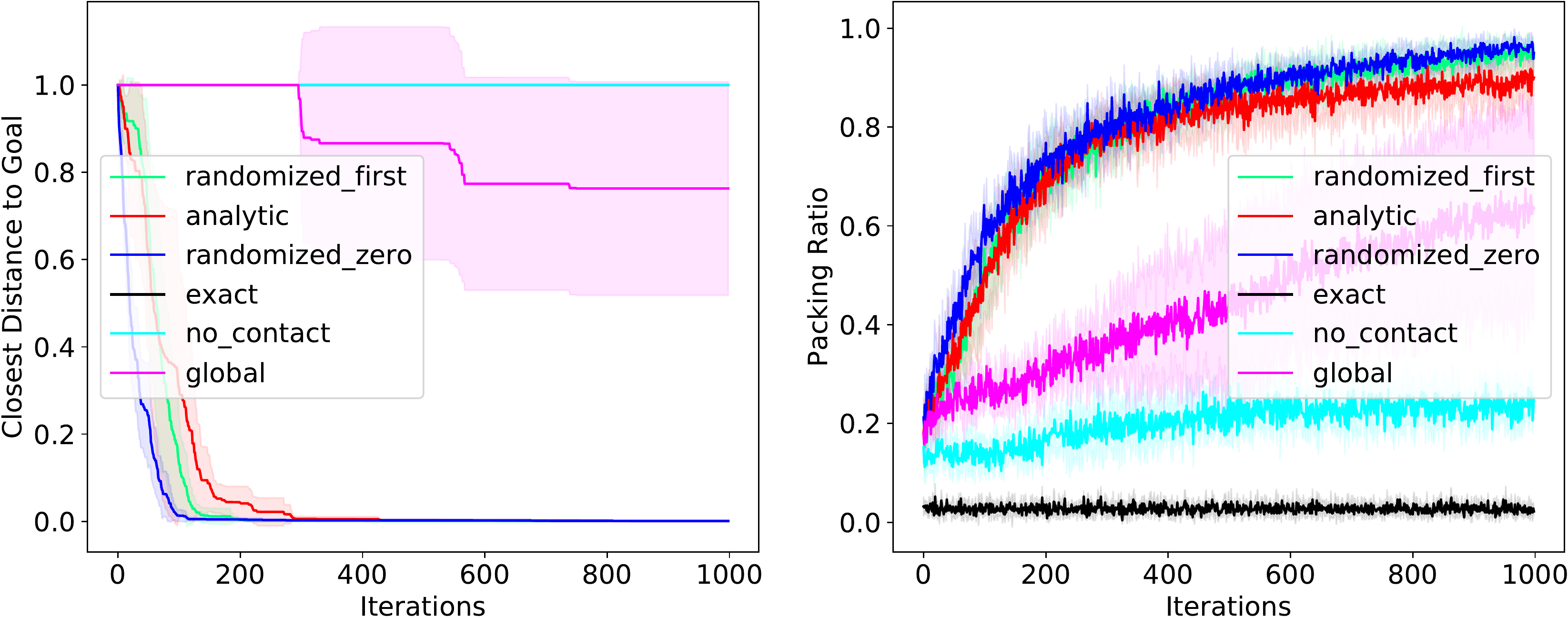}
}
\subfloat[Allegro Plate.]{
	\includegraphics[width=0.485\linewidth]{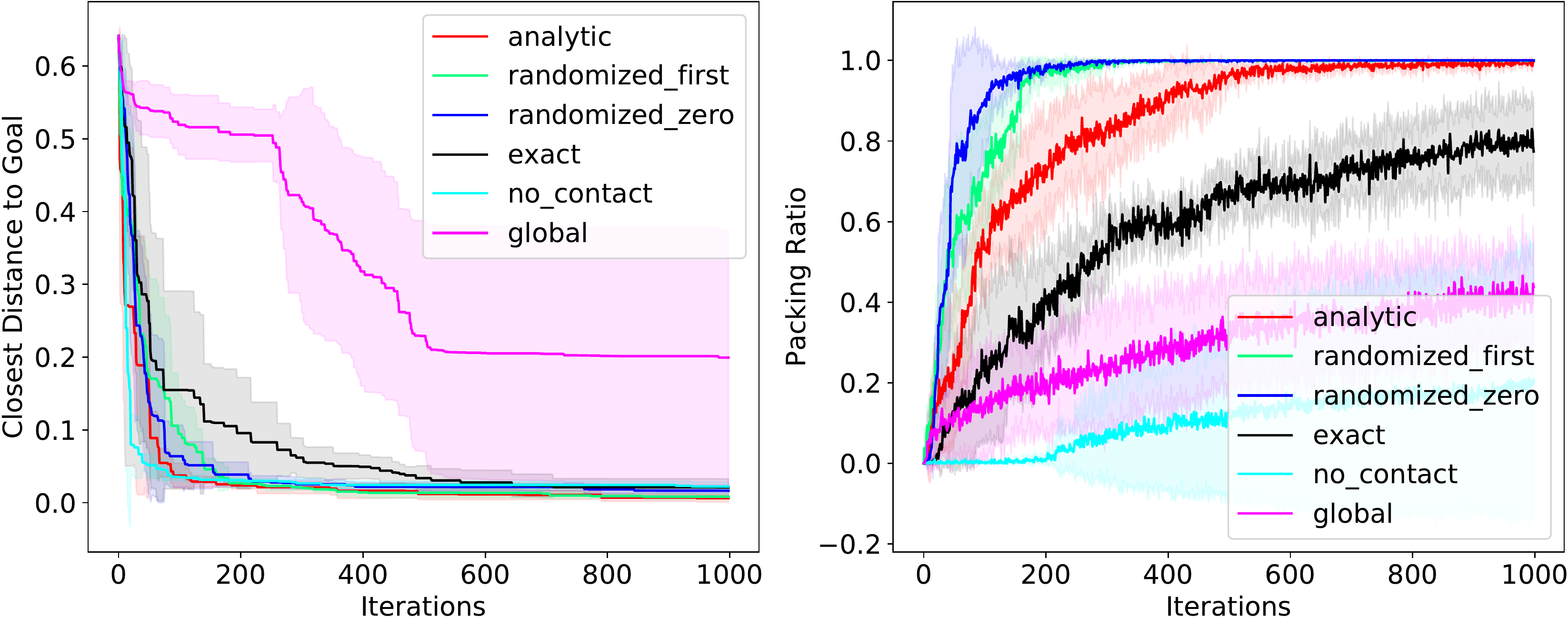}
} \\
\vspace{-0.2cm}
\subfloat[Planar Hand.]{
	\includegraphics[width=0.485\linewidth]{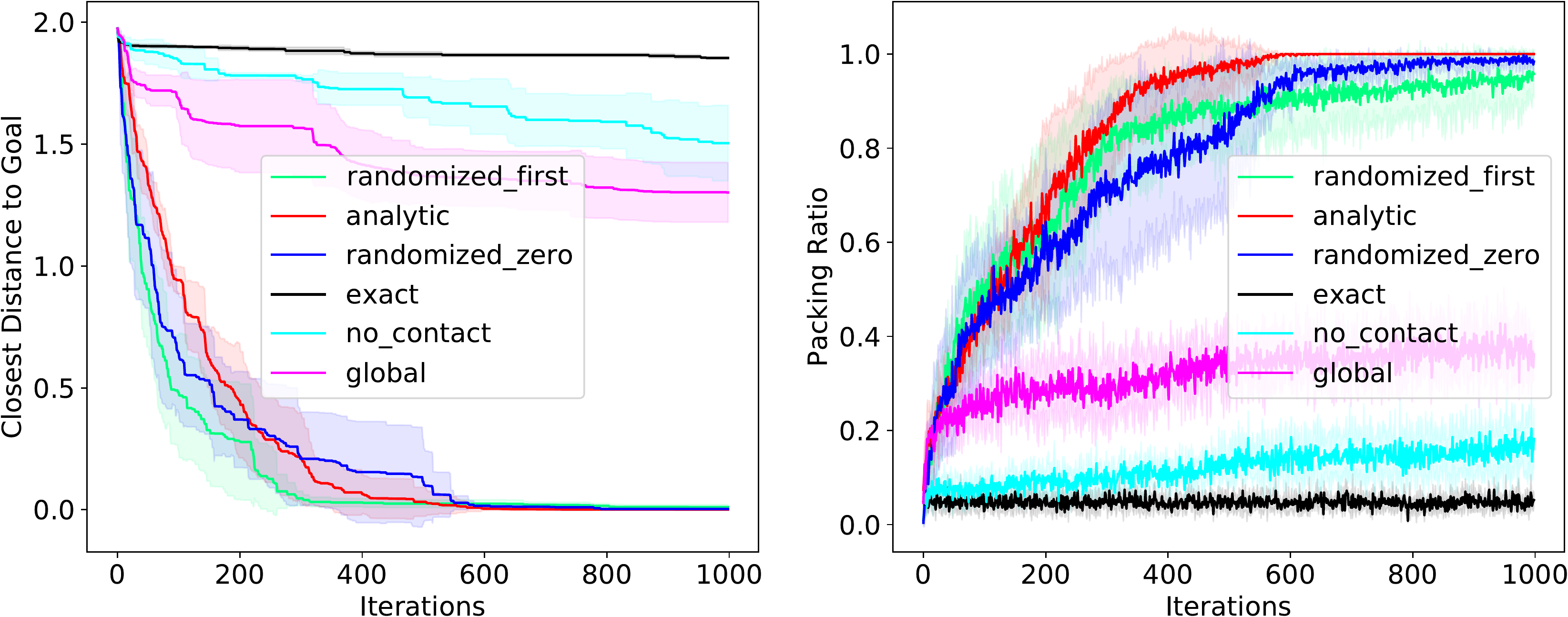}
}
\subfloat[Allegro Pen.]{
	\includegraphics[width=0.485\linewidth]{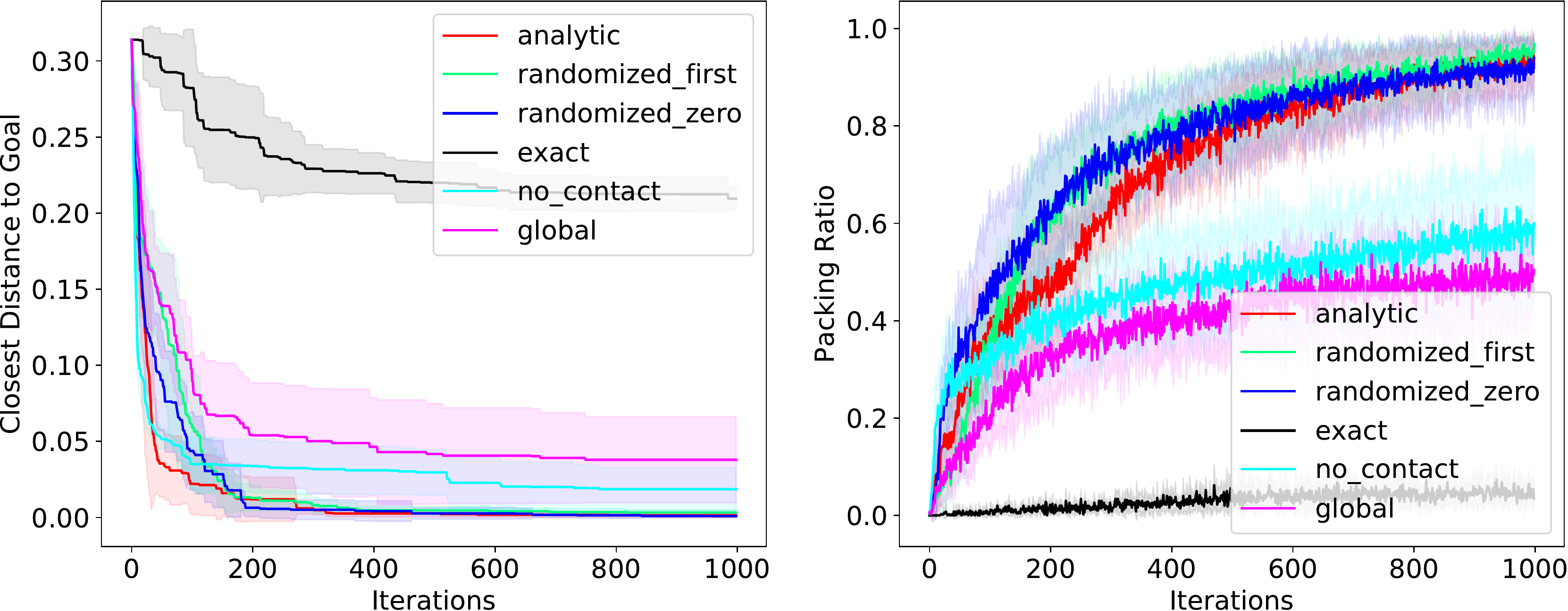}
} \\
\vspace{-0.2cm}
\subfloat[Allegro Hand.]{
	\includegraphics[width=0.485\linewidth]{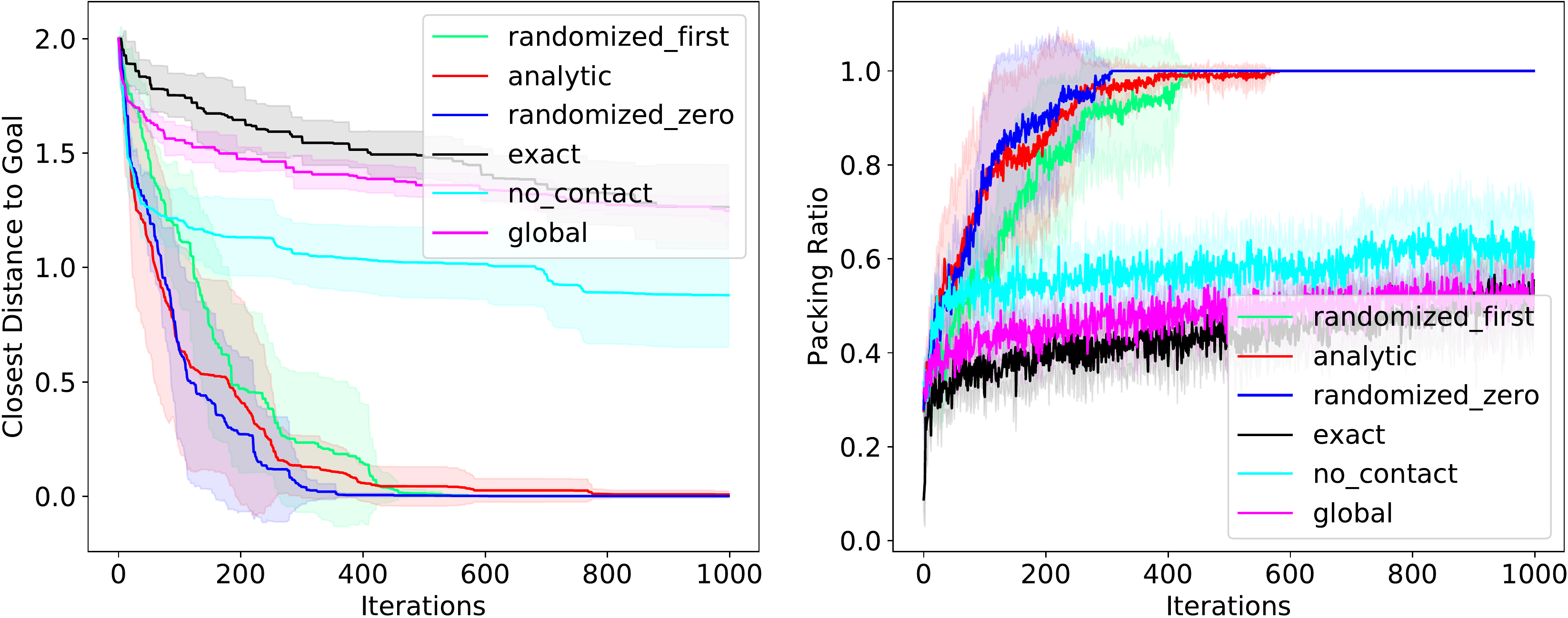}
}
\subfloat[Allegro Door.]{
	\includegraphics[width=0.485\linewidth]{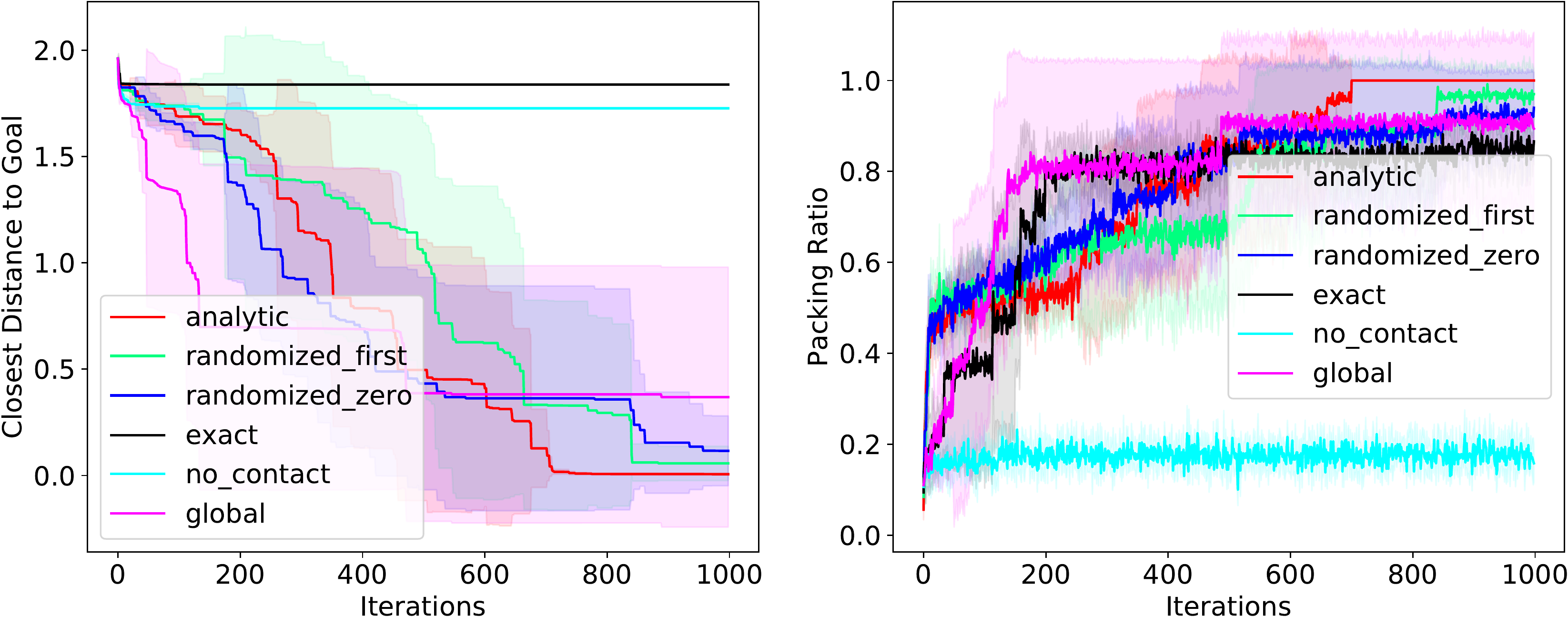}
}
\caption{Planning performance for the tasks in Fig. \ref{fig:rrt_tasks}. Results include running RRT with the enhancements proposed in Sec. \ref{sec:rrt_for_contact} using the three smoothing schemes from Sec. \ref{sec:smoothing_computation_schemes}, as well as the three ablation studies proposed in Sec. \ref{sec:rrt_experiment_setup}.}
\label{fig:rrtperformance}
\vskip -0.15 true in
\end{figure*}

\iffalse
\begin{table*}[thpb] \label{tab:trajoptresults}
\centering
\begin{tabular}{|| c | r | r | r | r | r | r || r | r | r | r | r | r || } 
 \hline
 Problem & \multicolumn{2}{||c||}{PlanarPushing} & \multicolumn{2}{||c||}{PlanarHand} & \multicolumn{2}{||c||}{AllegroHand} &
 \multicolumn{2}{||c||}{AllegroPlate} & 
 \multicolumn{2}{||c||}{AllegroPen} & 
 \multicolumn{2}{||c||}{AllegroDoor} 
 \\\hline
 Method & Cost & Time(s) & Cost & Time(s) & Cost & Time(s) & Cost & Time(s) & Cost & Time(s) & Cost & Time(s)  \\ \hline\hline

    Analytic && 3.25  && 10.32 && 31.01 && 117.16 && 24.84 && 12.42  \\
    Randomized  && 7.50 && 21.34 && 80.12 && 161.64 && 86.12 && 34.55\\
    RandomizedZero && 7.40 && 21.00 && 82.00 && 168.11 && 81.23 && 33.80 \\
    Exact && - && - && - && - && - && - \\
    NoContact && - && - && - && - && - && - \\
    Global && - && - && - && - && - && - \\\hline
\end{tabular}
\caption{Minimum cost and running time achieved by different methods. Number of iterations run. Every run except \code{AllegroDoor} was run for $1000$ iterations. We choose to not display running time for the ablation options since they are slight variations of Analytic with comparable running times.}
\label{table:rrtrresults}
\end{table*}
\fi 

For both metrics, the performance of our RRT algorithms averaged over 5 runs is plotted in Fig. \ref{fig:rrtperformance} for different smoothing schemes, as well as ablations of different choices we made for the algorithm. In particular, we run three variations of our algorithm while ablating a crucial ingredient.
\begin{enumerate}
    \item \textbf{Exact}. We replace the linearization of the smooth surrogate $\mathbf{B}_\rho$ with the exact linearization $\mathbf{B}$ of CQDC dynamics, used for both extension and metric computation.
    \item \textbf{NoContact}. We do not allow contact sampling (Sec.\ref{sec:contactsampling}) in this variant of the algorithm.
    \item \textbf{Global}. Instead of the local Mahalanobis metric, we use a globally-uniform metric during the $\Nearest$ step of the algorithm. For our experiments, we use a carefully-chosen weighted Euclidean norm. 
\end{enumerate}
% By showing the results, we aim to show that the choices we made in the design of the algorithm were necessary.

\subsection{Results \& Discussion}

We plot the results of our experiments in Fig. \ref{fig:rrtperformance}, and display the running time of our algorithm using the three different smoothing schemes in Table \ref{table:rrtrresults}. We discuss some of our findings from the experiment, in the context of our hypotheses in the beginning of this section.

\begin{table}[thpb]
\centering
\begin{tabular}{|| c | r | r | r | r | r | r || } 
 \hline
 Method & PPushing & PHand & AHand & APlate & APen & ADoor
 \\\hline
    A. & 3.25  & 10.32 & 31.01 & 117.16 & 24.84 & 12.42  \\
    RF.  & 7.50 & 21.34 & 80.12 & 161.64 & 86.12 & 34.55\\
    RZ. & 7.40 & 21.00 & 82.00 & 168.11 & 81.23 & 33.80 \\\hline
\end{tabular}
\caption{Running time achieved by different methods in seconds. Every trial was run for $1000$ iterations. We choose to not display running time for the ablation options since they are slight variations of Analytic with comparable running times.}
\label{table:rrtrresults}
\end{table}

\subsubsection{Smoothing vs. Exact} Throughout all experiments, we saw that using the exact linearization to compute the distance metric and extension results in much worse performance compared to any of the smoothing schemes, which supports our hypothesis that mode smoothing is necessary in order to solve many of the tasks.

\subsubsection{Analytic vs. Randomized Smoothing} For most of the tasks, we saw no meaningful difference between analytic and randomized smoothing schemes in terms of both how fast the goal is reached and the packing ratio. This empirically supports our theory that the two smoothing schemes are equivalent methods to compute local models of surrogate dynamics. The running time in Table \ref{table:rrtrresults}, however, shows that analytic smoothing results in faster computation time as it does not require taking multiple samples. 

\subsubsection{Global vs. Mahalanobis Metric} Despite reasonable efforts to choose good weights for the weighted Euclidean norm, we consistently observed that the globally uniform weighted Euclidean metric resulted in much worse performance compared to the local Mahalanobis metric. This supports our hypothesis, and the findings of \cite{shkolnik2009reachability} that kino-dynamic RRT in general greatly benefits from guiding tree growth with reachability information. 

\subsubsection{Effect of Contact Sampling} For some of the tasks (e.g. \code{AllegroPlate, AllegroPen}), contact sampling was not necessary. However, for examples that require resetting the actuator into a completely different configuration to make progress (e.g. \code{PlanarPushing, AllegroHand, PlanarHand}), contact sampling greatly improves the planner's performance. Although our contact samplers are simple and can benefit from more sophisticated implementations, the experiments show that contact sampling is essential for RRT to efficiently explore through contact dynamics constraints.

\section{Sim2Real Transfer \& Hardware Results}\label{sec:sim2real}
Although our planner successfully plans through our CQDC dynamics model, we further investigate if the plans can successfully transfer to real experiments. For this purpose, we run the obtained plans from Sec.\ref{sec:rrt_results} in \emph{open-loop} on a higher fidelity simulator Drake \cite{drake}, as well as an actual hardware setting. These experiments further shed light on the efficacy and the limitations of our proposed method. 

\begin{figure*}[thpb]
\centering\includegraphics[width = 0.9\linewidth]{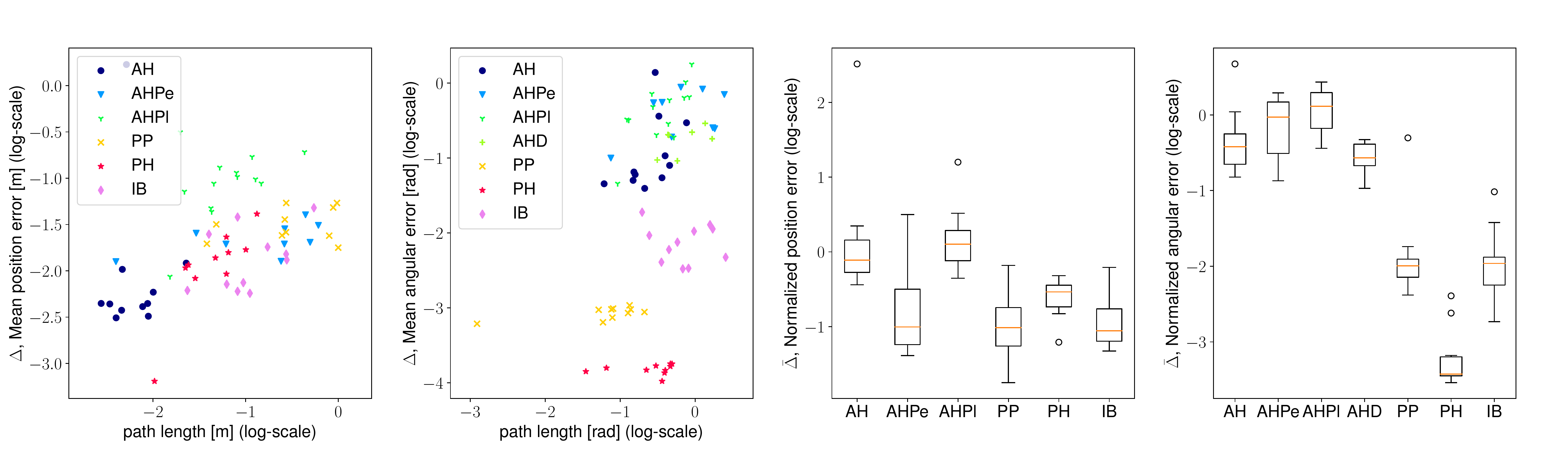}
\caption{Plots for sim2real performance of our CQDC dynamics, evaluated on the plans of Sec.\ref{sec:rrt_results}. \textbf{First Two Columns}: Scatter plot of mean error $\Delta$ vs. path length for position (first column) and orientation (second column). Each dot in the plot represents one segment trajectory. \textbf{Last Two Columns}: Box plot for the normalized error $\bar{\Delta}$ for positions (third column) and orientation (fourth column). Note that the mean in orange corresponds to the slope of the graph in the first two columns. Finally, we note that the \code{AllegroHandDoor} (AHD) example only consists of orientation, and has no position plot. Readers are highly encouraged to watch the accompanying video for the qualitative behavior of the actual segment trajectories in this plot. System abbreviations: AH(\code{AllegroHand}), AHPe(\code{AllegroPen}), AHPl (\code{AllegroPlate}), PP(\code{PlanarPushing}), PH(\code{PlanarHand}), AHD (\code{AllegroDoor}), IB(\code{IiwaBimanual})}
\label{fig:sim_to_real}
\vskip -0.2 true in
\end{figure*}

\subsection{Experiment Setup}

\subsubsection{Open-Loop Plan Transfer}
% @ Pang, do you think it's a good idea to add a subsection on how we took the CQDC plans (evaluated on h=0.1) and transferred it to the setting of MBP / real with ~100Hz rates? 
\newcommand{\qsimcoarse}{\{q_{k,\mathrm{sim}}\}^K_{k=0}}
\newcommand{\qsimfine}{\{q_{t,\mathrm{sim}}\}^T_{t=0}}
\newcommand{\ucoarse}{\{u_k\}^{K-1}_{k=0}}
\newcommand{\ufine}{\{u_t\}^{T-1}_{t=0}}
\newcommand{\hsim}{h_{\mathrm{sim}}}
\newcommand{\hreal}{h_{\mathrm{real}}}
\newcommand{\qreal}{\{q_{t,\mathrm{real}}\}^T_{t=0}}
\newcommand{\Tsim}{\mathcal{T}_{\mathrm{sim}}}
\newcommand{\qusim}{\{q^\mathrm{u}_{t,\mathrm{sim}}\}^T_{t=0}}
\newcommand{\qureal}{\{q^\mathrm{u}_{t,\mathrm{real}}\}^T_{t=0}}
\newcommand{\qusimt}{q^\mathrm{u}_{t,\mathrm{sim}}}
\newcommand{\qurealt}{q^\mathrm{u}_{t,\mathrm{real}}}

\newcommand{\qsimp}{q_{\mathrm{sim}}}
\newcommand{\qrealp}{q_{\mathrm{real}}}
\newcommand{\qusimp}{q^\mathrm{u}_{\mathrm{sim}}}
\newcommand{\qurealp}{q^\mathrm{u}_{\mathrm{real}}}

Our plan consists of state and action sequences, i.e. lists of ``knot'' points, that are consistent with the CQDC dynamics. We first divide this plan into individual segments $(\qsimcoarse, \ucoarse)$, punctuated by the $\ContactSample$ operation. 

Next, we convert the knot points $(\qsimcoarse, \ucoarse)$ into state and action trajectories $q_\mathrm{sim}: [0, T] \rightarrow \R[\nU + \nA]$ and $u: [0, T] \rightarrow \R[\nA]$ using first-order hold\footnote{we connect adjacent knot points with linear interpolation for positions and joint angles, or spherical linear interpolation (Slerp) for 3D orientations.}. Here $T$ denotes the duration of the trajectories in seconds. Rolling out $u(\cdot)$ on the real dynamics gives $q_\mathrm{real}: [0, T] \rightarrow \R[\nU + \nA]$, which is compared against $q_\mathrm{sim}(\cdot)$ to evaluate the sim2real performance.

\subsubsection{Evaluation Metrics}
To evaluate the performance of sim2real transfer, we first define the mean error $\Delta(\cdot,\cdot)$ between the two trajectories $q_\mathrm{sim}^\mathrm{u}(\cdot)$ and $q_\mathrm{real}^\mathrm{u}(\cdot)$ as
\begin{equation}
\Delta(q_\mathrm{sim}^\mathrm{u}, q_\mathrm{real}^\mathrm{u}) \coloneqq 
\frac{1}{T}
\int^T_{0} d\left(q_\mathrm{sim}^\mathrm{u}(t), q_\mathrm{real}^\mathrm{u}(t)\right) \mathrm{d}t
\end{equation}
where $d(\cdot,\cdot)$ is the Euclidean 2-norm for position (in meters), and the absolute change in angle for orientation (in radians). For 3D, this change of angle comes from axis-angle rotations. 

In addition, we expect that the metric $\Delta$ will depend on how much movement is inside the reference trajectory of the plan. To account for this scaling, we normalize $\Delta$ by dividing it by the length of the trajectory in the original plan, and denote the normalized error as $\bar{\Delta}$:
\begin{equation}
\bar{\Delta}(\qusimp, \qurealp) \coloneqq 
\frac{\Delta (q_\mathrm{sim}^\mathrm{u}, q_\mathrm{real}^\mathrm{u})}
{L(q_\mathrm{sim}^\mathrm{u})}
,
\end{equation}
where the denominator computes the path length of $q_\mathrm{sim}^\mathrm{u}$: 
\begin{equation}
L(q_\mathrm{sim}^\mathrm{u}) \coloneqq 
\int_0^T \norm{\Dot{q}_\mathrm{sim}^\mathrm{u}(t)}_2 \mathrm{d} t =
\sum^{K - 1}_{k=0}d(q^\mathrm{u}_{k+1,\mathrm{sim}}, q^\mathrm{u}_{k,\mathrm{sim}}).
\end{equation}

This normalization also takes into account the inherent scales of the system, and makes $\bar{\Delta}(\cdot,\cdot)$ a dimensionless quantity. For each system in Fig. \ref{fig:rrt_tasks}, we obtain at least $10$ segments and evaluate our error metrics. 

\subsubsection{Simulation Setup}
We transfer the examples of Fig.\ref{fig:rrt_tasks} into Drake \cite{drake}, which utilizes a full second-order dynamics model with a sophisticated and accurate contact solver \cite{tamsi}. The collision geometries, robot controller stiffness and coefficients of friction are kept consistent between the CQDC dynamics and Drake. 
\begin{figure}[thpb]
\centering\includegraphics[width = 1.0\linewidth]{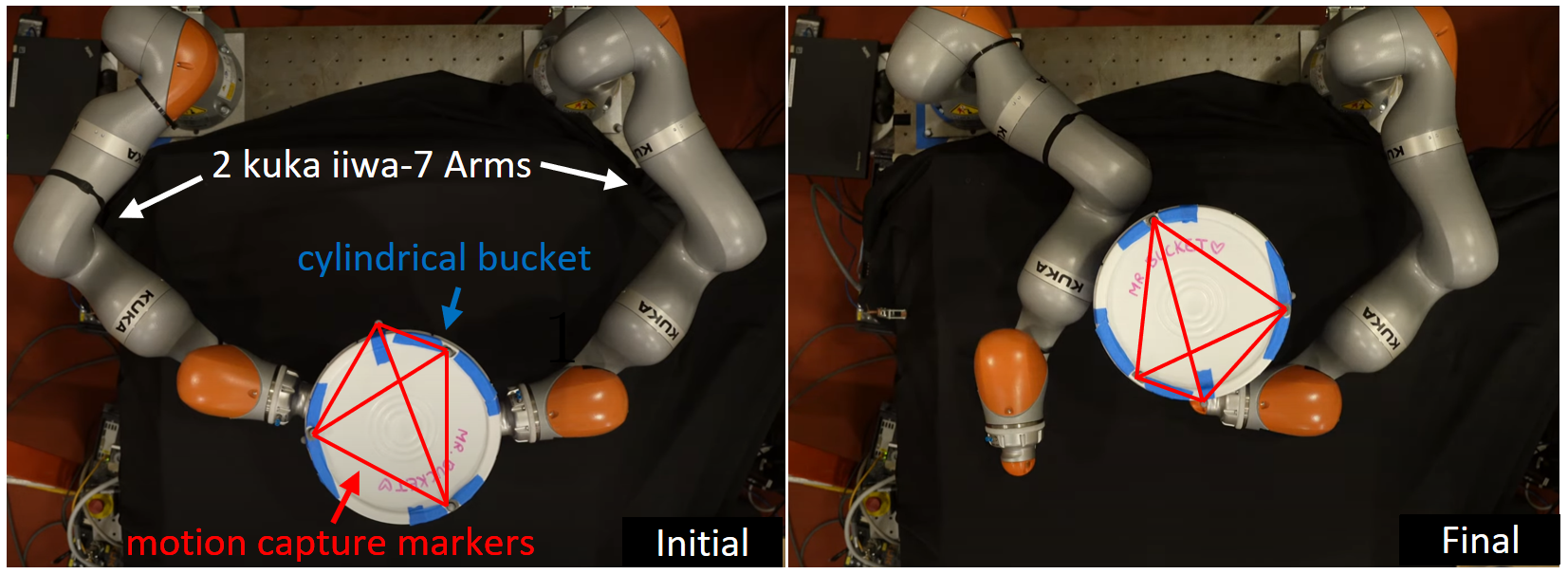}
\caption{Hardware for the \code{IiwaBimanual} setup, where the goal is to rotate the bucket by $180^\circ$. The left and right pictures correspond to the initial state and the final state after the open-loop plan execution. The lines between motion capture markers are connected to illustrate the change of pose in the bucket. Readers are encouraged to watch the accompanying video for the full execution.} 
\label{fig:hardware}
\end{figure}

\subsubsection{Hardware Setup}
To verify results on actual hardware, we create a variant of the \code{PlanarHand} environment, where the object is replaced by a bucket, and 2 Kuka iiwa arms are used for the actuators. We name this environment \code{IiwaBimanual}. We utilize a motion capture system to estimate the state of the bucket in order to compare the two trajectories of $\qusimp(\cdot)$ and $\qurealp(\cdot)$. Our setup is illustrated in Fig. \ref{fig:hardware}.

\subsection{Results \& Discussion}
We plot the results of our experiments in Fig.\ref{fig:sim_to_real}. While 2D systems such as \code{PlanarPushing}, \code{PlanarHand}, and \code{IiwaBimanual} display low error and good sim2real transfer, 3D systems such as \code{AllegroHand}, \code{AllegroPlate}, \code{AllegroPen} and \code{AllegroDoor} show larger error. To better understand the discrepancy of sim2real performance on different systems, we visualized trajectories from all systems by overlaying $\qurealp$ on top of $\qsimp$ (see accompanying video).  

From the visualizations, it is clear that on \emph{all} systems, there exists a persistent \emph{phase} difference between $\qurealp$ and $\qusimp$: $\qurealp$ tends to lag $\qusimp$ when the robot accelerates, and lead when the robot decelerates. This is not surprising, as the CQDC dynamics that generates $\qusimp$ is inherently a first-order system, whereas $\qurealp$ is generated from second-order dynamics. 
On trajectory segments with good sim2real performance, this only results in harmless oscillations of $\qurealp$ around $\qusimp$. In these cases, we believe that the good sim2real performance validates our contact model. We now discuss other cases where sim2real is not as effective.

\subsubsection{Violation of Quasi-dynamic Assumption} 
The quasi-dynamic assumption implies that objects are quickly stopped by damping when they are not ``pushed around'' by the robot. This is true on 2D systems, as the friction patch between the object and its supporting surface is persistent and always provides enough damping to bring the object to still. 

However, the necessary damping to uphold the quasi-dynamic assumption does not always exist on 3D systems. For instance, in \code{AllegroHand} and \code{AllegroHandPen}, the point contact between the object and the palm provides very little damping. Most of the damping comes from the finger joints when the fingers are opposing the object's motion. Therefore, when the grasp on the object is not tight, the object can roll quite far from the planned trajectory or even off the palm. 

\subsubsection{Missed Contacts}
Due to the non-smooth nature of contact dynamics, small discrepancies in object trajectory caused by the phase gap can lead to the robot completely missing contacts with the object. For example, we observed that in \code{AllegroHandPlate} or \code{AllegroHandDoor}, some grasps that were valid under the CQDC dynamics no longer succeeded in holding the object in place in Drake. The consequence of these failed grasps is that plates are dropped on the table in \code{AllegroHandPlate}, and door handles are missed in \code{AllegroHandDoor}.

\subsubsection{Necessity of Stabilization \& Robustification}
These results tell us that plans suggested by the CQDC dynamics can give high-level directions, but its open-loop execution may not succeed under second-order dynamics with high velocities and low amount of damping. We believe that tracking this high-level plan requires low-level feedback controllers that can stabilize to the plan, and actively enforce the closed-loop system to be quasi-dynamic. We also believe that the high-level planer can benefit from robustness objectives such as encouraging grasps that are considered good under classical grasping metrics. Such grasps will increase the control authority of the robot over the object, thereby providing sufficient damping and decreasing the chances of dropping the object. We leave these as promising directions for future work. 

\section{Conclusion}
\noindent We motivated our work by noting the stark contrast between the empirical effectiveness of RL for contact-rich settings and the struggle of the model-based methods. By identifying the pitfalls in the existing model-based methods for planning, understanding how RL was able to alleviate such pitfalls, and resolving them with model-based techniques, we have shown that traditional model-based approaches can be effective in tackling planning for contact-rich manipulation. Compared to existing tools in RL which use heavy offline computation in the order of hours or days, our contribution offers a powerful alternative in the spectrum of solutions by enabling efficient online planning in the order of a minute while being generalizable with respect to environments and tasks. We recap some of our contributions that enabled this process. 

We first pointed out the pitfall of explicit enumeration and consideration of modes in model-based methods, whereas RL has smoothed them out stochastically. Through Sec. \ref{sec:localsmoothing}, we formalized the process of smoothing for functions. This allowed us to view randomized smoothing in RL as a computational alternative to analytic smoothing used in existing model-based literature that both approximate local models of the smooth surrogate. This connection was made more concrete in the setting of contact dynamics in Sec. \ref{sec:smoothingequivalence}.  

Next, we have pointed out another weakness in model-based approaches where transients of second-order dynamics lead to myopic linearizations which are uninformative about long-term planning. As a solution to this weakness, we have presented CQDC, a novel formulation of contact model that is convex, differentiable, and quasi-dynamic. Through a number of theoretical arguments and empirical studies, we have shown the efficacy of our contact model. We have further shown that by inspecting the structure in our proposed model, we can analytically smooth out the contact dynamics with a log-barrier relaxation. With experiments, we have shown that our method of analytic smoothing has computational benefits over randomized smoothing.

Finally, we observed that existing model-based methods that rely on smoothing have been tied to local trajectory optimization. Due to their weakness to local minima, they have been less effective in difficult problems compared to RL-based approaches that attempt to perform global search. On the other hand, the SBMP methods for contact-rich systems have explicitly considered contact modes which fall into the pitfall of mode enumeration. Our contribution fills in a gap in existing methods by combining mode smoothing with RRT, where local approximation to the smooth surrogate was used to guide the exploration process of RRT via the local Mahalanobis metric.

By combining these three contributions, we have enabled efficient global motion planning for highly contact-rich and high-dimensional systems that were previously not achievable by existing model-based or RL-based methods. We believe that in the future, a highly optimized version of our planner can be used to perform real-time motion planning, or be used to guide policy search. With this capability, we hope to enable robots to find contact-rich plans online in previously unseen environments within seconds of planning time.

\bibliographystyle{IEEEtran}
\bibliography{references.bib}
\end{document}